\documentclass[12pt]{article}

\usepackage{blindtext}
\usepackage{amsthm}

\usepackage{microtype}
\usepackage{graphicx}
\usepackage{subfigure}
\usepackage{booktabs} 
\usepackage{amsmath} 

\usepackage[utf8]{inputenc}
\usepackage[T1]{fontenc}
\usepackage{hyperref}
\usepackage{url}
\usepackage{booktabs}
\usepackage{amsfonts}
\usepackage{amsmath}
\usepackage{amsthm}
\usepackage{amssymb} 
\usepackage{nicefrac}
\usepackage{microtype}
\usepackage{xcolor}
\usepackage{graphicx}
\usepackage{enumitem}
\usepackage{natbib} 

\usepackage{geometry} 

\theoremstyle{plain}
\newtheorem{theorem}{Theorem}[section]

\newtheorem{lemma}[theorem]{Lemma}

\theoremstyle{definition}

\newtheorem{assumption}[theorem]{Assumption}
\theoremstyle{remark}
\newtheorem{remark}[theorem]{Remark}

\newcommand{\NnOp}{\mathcal{N}} 
\newcommand{\FfOpTilde}{\tilde{\mathcal{F}}} 
\newcommand{\R}{\mathbb{R}}

\newcommand{\Qstar}{Q^*}
\newcommand{\Vstar}{V^*}

\newcommand{\StateSpace}{\mathcal{S}}
\newcommand{\ActionSpace}{\mathcal{A}}
\newcommand{\params}{\theta}
\newcommand{\BellmanOp}{\mathcal{B}}
\newcommand{\IdOp}{\mathcal{I}}
\newcommand{\ModBellmanOp}{\mathcal{J}} 
\newcommand{\supnorm}[1]{\|#1\|_{\infty}}
\newcommand{\QnnL}{\hat{Q}_{\text{NN}}^{(L)}} 
\newcommand{\indicator}[1]{\mathbf{1}_{#1}} 

\newcommand{\LipConstH}{\Lambda_h}
\newcommand{\LipConstSigma}{\Lambda_\sigma} 
\newcommand{\LipConstR}{\Lambda_r}
\newcommand{\LipConstG}{\Lambda_g}
\newcommand{\BoundH}{M_h}
\newcommand{\BoundSigma}{M_\sigma}
\newcommand{\BoundR}{M_r}
\newcommand{\BoundG}{M_g}
\newcommand{\LunifLip}{L_{\text{unif-Lip}}} 
\newcommand{\LhatUnifLip}{L^*_{\text{unif-Lip}}} 
\newcommand{\MhatUnifBound}{M^*_{\text{unif-bound}}} 
\newcommand{\LFstar}{L_F^*} 

\title{Universal Approximation Theorem for Deep Q-Learning via FBSDE System}

\author{%
  Qian Qi\thanks{School of Computer Science, Peking University, Beijing, 100871, P.R.China. Email: \texttt{qiqian@pku.edu.cn}. }
}

\date{} 

\begin{document}

\maketitle

\begin{abstract}
The approximation capabilities of Deep Q-Networks (DQNs) are commonly justified by general Universal Approximation Theorems (UATs) that do not leverage the intrinsic structural properties of the optimal Q-function, the solution to a Bellman equation. This paper establishes a UAT for a class of DQNs whose architecture is designed to emulate the iterative refinement process inherent in Bellman updates. A central element of our analysis is the propagation of regularity: while the transformation induced by a single Bellman operator application exhibits regularity, for which Backward Stochastic Differential Equations (BSDEs) theory provides analytical tools, the uniform regularity of the entire sequence of value iteration iterates—specifically, their uniform Lipschitz continuity on compact domains under standard Lipschitz assumptions on the problem data—is derived from finite-horizon dynamic programming principles. We demonstrate that layers of a deep residual network, conceived as neural operators acting on function spaces, can approximate the action of the Bellman operator. The resulting approximation theorem is thus intrinsically linked to the control problem's structure, offering a proof technique wherein network depth directly corresponds to iterations of value function refinement, accompanied by controlled error propagation. This perspective reveals a dynamic systems view of the network's operation on a space of value functions.
\end{abstract}

\section{Introduction}
\label{sec:introduction}

Deep Reinforcement Learning (DRL) has achieved remarkable breakthroughs, with Deep Q-Networks (DQNs, see \cite{mnih2015human}) being a cornerstone. DQNs approximate the optimal action-value function, $\Qstar$, using deep neural networks, enabling agents to learn effective policies in high-dimensional state spaces. The theoretical underpinnings of why DQNs can successfully represent $\Qstar$ often rely on general Universal Approximation Theorems (UATs) for neural networks \cite{cybenko1989approximation,hornik1991approximation}. These theorems state that sufficiently large networks can approximate any continuous function on a compact set.

However, $\Qstar$ is not an arbitrary continuous function; it is the unique fixed point of a Bellman optimality operator, inheriting a rich structure from the underlying Markov Decision Process (MDP) dynamics, reward function, and discount factor. Standard UATs typically do not exploit this problem-specific structure. Moreover, while deep architectures like Residual Networks (ResNets, see \cite{he2016deep,han2019mean,Li2022Deep,qian2025}) are often employed and their depth is empirically crucial, the connection between depth and the approximation of $\Qstar$ is often qualitative (see \cite{qian2025}). 

In continuous-time stochastic control, the optimal value function $V^*$ is known to be the (often unique) viscosity solution to a Hamilton-Jacobi-Bellman (HJB) PDE. This PDE frequently has a probabilistic representation via Backward Stochastic Differential Equations (BSDEs) when coupled with the forward state process (forming an FBSDE system, e.g., \cite{peng1992stochastic,el1997backward,YongZhou99,ma1999forward}). These BSDE representations are known to yield important regularity properties (e.g., Lipschitz continuity) for the value function under appropriate assumptions on the coefficients of the dynamics and costs \cite{pardoux1999bsdes,fleming2006controlled}.

This paper aims to address this disparity by developing a UAT for a class of DQNs whose architecture is designed to reflect the iterative nature of Bellman updates. The proof technique leverages the propagation of regularity properties for the iterates of the Bellman operator. The transformation effected by a single Bellman step can be related to solving a short-horizon problem, whose value structure (and thus regularity) can be analyzed using tools related to BSDE theory. Crucially, the uniform regularity of all value iteration iterates (and the limit $\Qstar$)—specifically, uniform Lipschitz continuity over a finite horizon—is a consequence of dynamic programming principles. Our contributions are:
\begin{enumerate}
    \item We frame the approximation of $\Qstar$ (the solution to a $\delta$-discretized Bellman equation) not merely as function approximation, but as learning the limit of a dynamical system defined by the Bellman operator on a function space.
    \item We propose a DQN architecture where individual layers (or blocks of layers) are structured as neural operators. Each such operator block aims to approximate one step of the Bellman iteration, transforming a representation of the current Q-function iterate. Specifically, a block implementing $Q \mapsto Q + \FfOpTilde(Q)$ aims for $\FfOpTilde(Q)$ to approximate $\BellmanOp Q - Q$, where $\BellmanOp$ is the Bellman operator. This architecture establishes a structural correspondence between network depth and iteration count.
    \item We establish that the approximability of the Bellman operator (or its residual form $\ModBellmanOp = \BellmanOp - \IdOp$) by a neural operator block is predicated on the regularity of the functions it acts upon. Under standard Lipschitz and boundedness assumptions on the MDP coefficients (Assumption \ref{assump:mdp_coeffs}), the iterates $Q^{(k)}$ of the Bellman operator converging to $\Qstar$ are shown to be uniformly Lipschitz continuous and uniformly bounded on their compact domain. This property ensures that the iterates reside within a compact subset of $C(K_Q)$, a critical prerequisite for the application of neural operator UATs.
    \item Leveraging this uniform regularity, we develop a novel UAT proof where the network's approximation of $\Qstar$ is achieved through a sequence of transformations mirroring the Bellman iterations. This iterative refinement approach, with controlled layer-wise error accumulation, provides a problem-aware justification for the approximation power of these specialized DQNs and demonstrates stable error propagation, contingent upon the neural operators satisfying certain stability properties (see Assumption \ref{assump:neural_operator_class}).
\end{enumerate}

This work offers a new perspective on the approximation capabilities of certain DQNs, particularly by highlighting how control-theoretic regularity ensures that the sequence of Bellman iterates forms a tractable set for approximation by neural operators, enabling an approximation framework structured around iterative refinement.

The paper is organized as follows. Section \ref{sec:preliminaries} introduces the continuous-time MDP, defines $Q^*$ and $V^*$, and briefly reviews BSDEs and the proposed DQN architecture. Section \ref{sec:q_fbsde_scheme} outlines the core idea of representing $Q^*$ through an iterative scheme with regularity arguments. Section \ref{sec:fbsde_uat} presents the main approximation theorem and a sketch of its proof, supported by key lemmas. Section \ref{sec:quantitative_rates} discusses potential for deeper results concerning quantitative approximation rates and the curse of dimensionality, and Section \ref{sec:conclusion} concludes. Appendices \ref{app:proof_bellman_op_props_fortified} through \ref{app:neural_operators_details} provide detailed proofs for the main lemmas and theorem, and further discussion on neural operators.

\section{Preliminaries}
\label{sec:preliminaries}

\subsection{Continuous-Time Markov Decision Process}
Let $(\Omega, \mathcal{F}, (\mathcal{F}_u)_{u \in [0,T]}, \mathbb{P})$ be a filtered probability space supporting a $d$-dimensional standard Brownian motion $W = (W_u)_{u \in [0,T]}$, where $T < \infty$ is a finite time horizon. We assume the filtration $(\mathcal{F}_u)$ is the natural filtration generated by $W$, augmented to satisfy the usual conditions.
The state $s_u \in \StateSpace \subseteq \R^n$ starting from $s_t=x$ at time $t$ evolves according to:
\begin{equation} \label{eq:sde_state}
    ds_u = h(u, s_u, a_u) du + \sigma(u, s_u, a_u) dW_u, \quad s_t = x, \quad u \in [t,T],
\end{equation}
where $a_u \in \ActionSpace \subset \R^m$ is the action, chosen from a compact set $\ActionSpace$. We consider policies that select an action $a_t$ at time $t$ (adapted to $\mathcal{F}_t$) and hold it constant for a small duration $\delta > 0$, i.e., $a_u = a_t$ for $u \in [t, \min(t+\delta,T))$.
The functions $h: [0,T] \times \StateSpace \times \ActionSpace \to \R^n$ (drift) and $\sigma: [0,T] \times \StateSpace \times \ActionSpace \to \R^{n \times d}$ (diffusion) are specified in Assumption \ref{assump:mdp_coeffs}. The running reward rate is $r(t,s,a)$, and the terminal reward is $g(s)$. The continuous-time discount rate is $\lambda > 0$.

Let $K_S = [0,T] \times \StateSpace$ be the compact state-time domain, and $K_Q = K_S \times \ActionSpace = [0,T] \times \StateSpace \times \ActionSpace$ be the compact state-time-action domain. We consider functions $Q: K_Q \to \R$. Let $C(K_Q)$ be the space of continuous functions on $K_Q$, equipped with the supremum norm $\supnorm{f} = \sup_{(t,s,a) \in K_Q} |f(t,s,a)|$.

\begin{assumption}[MDP Coefficients and Rewards] \label{assump:mdp_coeffs}
    The state space $\StateSpace \subseteq \R^n$ and action space $\ActionSpace \subseteq \R^m$ are compact sets. The time horizon $T$ is finite.
    Let $X = (t,s,a)$ and $X' = (t',s',a')$ be points in $K_Q = [0,T] \times \StateSpace \times \ActionSpace$. We define the metric $d_{K_Q}(X,X') = |t-t'| + \|s-s'\| + \|a-a'\|$ (using Euclidean norms for $s,a$).
    There exist constants $\LipConstH, \LipConstSigma, \LipConstR, \LipConstG > 0$ and $\BoundH, \BoundSigma, \BoundR, \BoundG > 0$ such that for all $X, X' \in K_Q$:
    \begin{itemize}
        \item The functions $h(X)$, $\sigma(X)$, and $r(X)$ are uniformly Lipschitz continuous on $K_Q$. Specifically:
            \begin{itemize}
                \item $\|h(X) - h(X')\| \le \LipConstH d_{K_Q}(X,X')$.
                \item $\|\sigma(X) - \sigma(X')\|_F \le \LipConstSigma d_{K_Q}(X,X')$. (Frobenius norm)
                \item $|r(X) - r(X')| \le \LipConstR d_{K_Q}(X,X')$.
            \end{itemize}
        They are also bounded: $\|h(X)\| \le \BoundH$, $\|\sigma(X)\|_F \le \BoundSigma$, $|r(X)| \le \BoundR$. (Boundedness follows from Lipschitz continuity on a compact domain but stated for explicitness).
        \item The terminal reward function $g: \StateSpace \to \R$ is uniformly Lipschitz continuous on $\StateSpace$: $|g(s)-g(s')| \le \LipConstG \|s-s'\|$. It is also bounded: $|g(s)| \le \BoundG$.
        \item The linear growth condition often assumed for SDE existence on unbounded domains, e.g., $\|h(t,s,a)\| + \|\sigma(t,s,a)\|_F \le K_L(1+\|s\|)$, is satisfied in a bounded form due to the compactness of $\StateSpace$.
    \end{itemize}
\end{assumption}
\begin{remark}[On Assumption \ref{assump:mdp_coeffs}] \label{rem:assump_strength}
The uniform Lipschitz continuity of $h, \sigma, r$ with respect to $(t,s,a)$ and $g$ with respect to $s$ (Assumption \ref{assump:mdp_coeffs}) is crucial for ensuring that the optimal Q-function $\Qstar$ and all its Bellman iterates $Q^{(k)}$ are uniformly Lipschitz continuous on the compact domain $K_Q$. This regularity is fundamental for the subsequent compactness arguments (Lemma \ref{lem:compactness_iterates}) and the applicability of neural operator UATs (Lemma \ref{lem:neural_op_uat}).\footnote{While strong, these assumptions are standard in stochastic control theory for establishing such regularity (e.g., \cite{fleming2006controlled}). Relaxing these assumptions (e.g., to Hölder continuity or local Lipschitz conditions if domains were unbounded) would significantly complicate the regularity analysis and is beyond the scope of the current work, though an important direction for broader applicability. A detailed proof of Lemma \ref{lem:bellman_op_props} (d) (see Appendix \ref{app:proof_bellman_op_props_fortified}) explicitly shows how the uniform Lipschitz constant $\LunifLip$ depends on the constants in Assumption \ref{assump:mdp_coeffs} (such as $\LipConstH, \BoundH$, etc.), $T, \lambda, \delta$.}
\end{remark}

The optimal action-value function $\Qstar(t,s,a)$ for the problem where controls are held constant for duration $\delta$, is the unique fixed point in $C(K_Q)$ of the Bellman optimality equation:
\begin{align} \label{eq:bellman_q_discrete}
\Qstar(t,s,a) = \mathbb{E} \bigg[ \int_t^{\min(t+\delta,T)} e^{-\lambda(\tau-t)} r(\tau,s_\tau,a) d\tau + \indicator{t+\delta \le T} e^{-\lambda \delta} \Vstar(t+\delta, s_{t+\delta}) \nonumber \\
+ \indicator{t+\delta > T} e^{-\lambda (T-t)} g(s_T) \, \bigg| \, s_t=s, a_t=a \bigg],
\end{align}
where action $a$ is applied over $[t, \min(t+\delta,T))$, and $\Vstar(u,x) = \sup_{a' \in \ActionSpace} \Qstar(u,x,a')$. $\indicator{}$ is the indicator function. For simplicity in some discussions, we might write $t+\delta$ assuming $t+\delta \le T$, but the full definition \eqref{eq:bellman_q_discrete} handles the terminal boundary at $T$. We assume $T$ is a multiple of $\delta$ for notational simplicity in iterative schemes where depth corresponds to time steps, but the definition of $\Qstar$ holds generally. The function $\Qstar$ is the value function for this specific $\delta$-discretized control problem structure, representing the optimal expected discounted future reward over the entire horizon $[t,T]$.

The function $V(t,s)$ for a fully continuous control problem (not necessarily identical to $\Vstar$ derived from Eq. \eqref{eq:bellman_q_discrete}) is the unique continuous viscosity solution to the Hamilton-Jacobi-Bellman (HJB) equation:
\begin{equation} \label{eq:hjb}
    -\frac{\partial V}{\partial t} - \sup_{a \in \ActionSpace} \{ \mathcal{L}^a V(t,s) + r(t,s,a) \} + \lambda V(t,s) = 0, \quad V(T,s) = g(s),
\end{equation}
where $\mathcal{L}^a V = \langle \nabla_s V, h(t,s,a) \rangle + \frac{1}{2} \text{Tr}(\sigma(t,s,a)\sigma(t,s,a)^T \nabla^2_{ss} V)$.\footnote{The HJB equation is typically interpreted in the viscosity sense, as $V$ may not be $C^{1,2}$ everywhere. Our main analysis focuses on the Bellman equation \eqref{eq:bellman_q_discrete} for $\Qstar$. Under Assumption \ref{assump:mdp_coeffs}, $\Qstar$ is Lipschitz continuous (see Lemma \ref{lem:bellman_op_props} (d) and its proof in Appendix \ref{app:proof_bellman_op_props_fortified}). For existence and uniqueness of viscosity solutions to HJB equations under Lipschitz conditions on coefficients, see \cite{fleming2006controlled}.}

\subsection{Forward-Backward Stochastic Differential Equations (FBSDEs)}
The solution $V(t,s)$ to the HJB equation \eqref{eq:hjb} can be characterized via BSDEs. If $\pi^*(u,s_u)$ is an optimal feedback control for the continuous problem, then $Y_u = V(u,s_u^{\pi^*})$ (where $s_u^{\pi^*}$ is the state under optimal control) solves the BSDE:
\begin{equation} \label{eq:bsde_vstar}
-dY_u = [r(u,s_u^{\pi^*},\pi^*(u,s_u^{\pi^*})) - \lambda Y_u] du - Z_u dW_u, \quad Y_T = g(s_T^{\pi^*}).
\end{equation}
Here $Z_u$ is related to $\nabla_s V(u,s_u^{\pi^*}) \sigma(u,s_u^{\pi^*}, \pi^*(u,s_u^{\pi^*}))$. Such BSDE representations are fundamental in stochastic control theory \cite{peng1992stochastic,el1997backward,YongZhou99} and are key to establishing regularity of $V$. More direct relevance for this paper comes from using BSDEs to characterize the transformation performed by one step of the Bellman iteration for Eq. \eqref{eq:bellman_q_discrete}, as detailed in Section \ref{sec:q_fbsde_scheme}. This characterization helps establish the necessary regularity (Lipschitz continuity) of the function $(\BellmanOp Q_c)$ given regularity of $Q_c$.

\subsection{Deep Q-Network Architecture (Operator-based ResNet)}
\label{sec:dqn_architecture}
We consider a Deep Q-Network that produces a sequence of Q-function approximations $\hat{Q}^{(l)} \in C(K_Q)$.
Let $D_M = \{p_j\}_{j=1}^M \subset K_Q$ be a finite discretization (grid) of $K_Q$.
Let $\mathcal{E}_M: C(K_Q) \to \R^M$ be an encoding operator, $\mathcal{E}_M(Q) = (Q(p_1), \dots, Q(p_M))$ for $p_j \in D_M$.
Let $\mathcal{D}_M: \R^M \to C(K_Q)$ be a decoding operator (e.g., multilinear interpolation, kernel interpolation) that reconstructs a continuous function from its values on $D_M$.
The network aims to learn $\Qstar$. Let $\hat{Q}^{(0)} \in C(K_Q)$ be an initial estimate (e.g., $\hat{Q}^{(0)} \equiv 0$).
The network consists of $L$ blocks, using a ResNet-like structure:
\begin{equation} \label{eq:dqn_resnet_operator_layer}
    \hat{Q}^{(l+1)} = \hat{Q}^{(l)} + \FfOpTilde_{\params_l}(\hat{Q}^{(l)}), \quad l=0, \dots, L-1.
\end{equation}
Here, $\FfOpTilde_{\params_l}: C(K_Q) \to C(K_Q)$ is the function realized by the $l$-th neural operator block. It is implemented as $\FfOpTilde_{\params_l}(Q) = \mathcal{D}_M(\NnOp_{\params_l}(\mathcal{E}_M(Q)))$, where $\NnOp_{\params_l}: \R^M \to \R^M$ is a neural network (e.g., an MLP) parameterized by $\params_l$.
This block aims for $\FfOpTilde_{\params_l}(\hat{Q}^{(l)})$ to approximate $\ModBellmanOp(\hat{Q}^{(l)}) = \BellmanOp \hat{Q}^{(l)} - \hat{Q}^{(l)}$, where $\BellmanOp$ is the Bellman operator (defined in Sec. \ref{sec:q_fbsde_scheme}). Thus, $\hat{Q}^{(l+1)}$ aims to represent $\BellmanOp \hat{Q}^{(l)}$.
The final Q-value approximation is $\QnnL = \hat{Q}^{(L)}$. The overall parameters of the DQN are $\params = (\params_0, \dots, \params_{L-1})$, and any parameters in $\mathcal{E}_M, \mathcal{D}_M$ if they are learned.\footnote{This architecture is specialized. The operator nature of $\FfOpTilde_{\params_l}$ acting on functions (approximated via $\NnOp_{\params_l}$ on discretized representations) is crucial. Each layer refines the entire Q-function estimate. Architectures such as DeepONets \cite{lu2021learning} or Fourier Neural Operators \cite{li2021fourier} provide frameworks for such operator learning, though here $\NnOp_{\params_l}$ acts on finite-dimensional vectors representing function evaluations.}

\section{Optimal Q-function via an Iterative Scheme with BSDE-Inspired Regularity}
\label{sec:q_fbsde_scheme}

Define the Bellman operator $\BellmanOp: C(K_Q) \to C(K_Q)$:
\begin{align} \label{eq:bellman_operator}
    (\BellmanOp Q)(t,s,a) = \mathbb{E}&\left[ \int_t^{\min(t+\delta,T)} e^{-\lambda(\tau-t)} r(\tau,s_\tau,a)d\tau + \indicator{t+\delta \le T} e^{-\lambda\delta} \nonumber\right.\\&\left.\sup_{a' \in \ActionSpace} Q(t+\delta, s_{t+\delta}, a') + \indicator{t+\delta > T} e^{-\lambda(T-t)}g(s_T) \bigg| s_t=s, a_t=a \right],
\end{align}
where action $a$ is fixed over $[t, \min(t+\delta,T))$, $s_\tau$ evolves via Eq. \eqref{eq:sde_state}.
The optimal Q-function $\Qstar$ is the unique fixed point of $\BellmanOp$.
Consider the sequence: $Q^{(0)} \in C(K_Q)$ (e.g., $Q^{(0)} \equiv 0$), and $Q^{(k+1)} = \BellmanOp Q^{(k)}$. This sequence defines a discrete-time dynamical system on the Banach space $C(K_Q)$, representing value iteration.

The computation of $(\BellmanOp Q_c)(t,s,a)$ for a given $Q_c \in C(K_Q)$ can be related to a BSDE. Let $(s_u)_{u \in [t,\min(t+\delta,T)]}$ be the forward state process (Eq. \eqref{eq:sde_state}) from $s_t=s$ under fixed action $a$. The BSDE for $Y_u$ on $u \in [t, \min(t+\delta,T)]$ is:
\begin{align}
-dY_u &= [r(u,s_u,a) - \lambda Y_u] du - Z_u dW_u, \quad u \in [t, \min(t+\delta,T)] \label{eq:bsde_bellman_step_Y} \\
  Y_{\min(t+\delta,T)} &= \begin{cases} \sup_{a' \in \ActionSpace} Q_c(t+\delta, s_{t+\delta}, a') & \text{if } t+\delta \le T \\ g(s_T) & \text{if } t+\delta > T \end{cases}. \label{eq:bsde_bellman_step_terminal}
\end{align}
Then, $Y_t = (\BellmanOp Q_c)(t,s,a)$. The existence, uniqueness, and regularity of $Y_t$ (as a function of $(t,s,a)$ and properties of $Q_c$) are standard results from BSDE theory. The driver $f(u,y,z; s_u,a) = r(u,s_u,a) - \lambda y$. Under Assumption \ref{assump:mdp_coeffs}, $r$ is Lipschitz in $(u,s_u,a)$, and $f$ is Lipschitz in $y$ (with coefficient $\lambda$) and inherits Lipschitz continuity in $(u,s_u,a)$ from $r$. The driver $f$ does not depend on $z$. If $Q_c$ is Lipschitz continuous in its arguments $(t',s',a')$ on $K_Q$, then its value $\sup_{a'} Q_c(t+\delta, s_{t+\delta}, a')$ is Lipschitz in $(t+\delta, s_{t+\delta})$ because $\ActionSpace$ is compact. Consequently, the terminal condition \eqref{eq:bsde_bellman_step_terminal} is Lipschitz in $s_{t+\delta}$ (or $s_T$) and depends regularly on $t+\delta$. BSDE theory (e.g., \cite{pardoux1999bsdes} or \cite{el1997backward} for Lipschitz properties of $Y_t$ with respect to initial data $(t,s)$ and parameters $(a)$, given a Lipschitz driver and a Lipschitz terminal condition w.r.t. the forward state $s_{\min(t+\delta,T)}$) implies that $Y_t$ will inherit Lipschitz continuity with respect to $(s,a)$ and appropriate regularity with respect to $t$, provided $Q_c$ is sufficiently regular. This BSDE perspective provides intuition for the regularity preservation of a single step $\BellmanOp Q_c$. The uniform regularity for all value iteration iterates $Q^{(k)}$ and $\Qstar$ (Lemma \ref{lem:bellman_op_props} (d)) is then established via an inductive argument on the finite horizon structure. We now state key properties of $\BellmanOp$ and the iterates $Q^{(k)}$.

\begin{lemma}[Properties of $\BellmanOp$ and Iterates $Q^{(k)}$] \label{lem:bellman_op_props}
Let Assumption \ref{assump:mdp_coeffs} hold. Let $K_Q$ be compact. Assume $\lambda > 0$.
\begin{enumerate}[label=(\alph*)]
    \item $\BellmanOp$ maps $C(K_Q)$ to $C(K_Q)$. If $Q \in C(K_Q)$ is also Lipschitz continuous on $K_Q$ (with Lipschitz constant $L_Q$), then $\BellmanOp Q$ is also Lipschitz continuous on $K_Q$. Its Lipschitz constant $L_{\BellmanOp Q}$ depends on $L_Q$, $\delta$, $\lambda$, the constants from Assumption \ref{assump:mdp_coeffs}, and $T$. (This is established in detail as part of (d)).
    \item $\BellmanOp$ is a contraction mapping on $(C(K_Q), \supnorm{\cdot})$ with contraction factor $L_{\BellmanOp} = e^{-\lambda\delta} < 1$.
    \item The value iteration sequence $Q^{(k+1)} = \BellmanOp Q^{(k)}$ with $Q^{(0)} \in C(K_Q)$ (e.g., $Q^{(0)} \equiv 0$) converges uniformly to the unique fixed point $\Qstar$ on $K_Q$.
    \item If $Q^{(0)}$ is Lipschitz continuous (e.g., $Q^{(0)} \equiv 0$), then $\Qstar$ is Lipschitz continuous, and all iterates $Q^{(k)}$ are uniformly bounded and uniformly Lipschitz continuous on $K_Q$. That is, there exists a constant $\LunifLip < \infty$, independent of $k$, depending on $T, \delta, \lambda$, the constants in Assumption \ref{assump:mdp_coeffs} (and implicitly the diameter of $K_Q$), such that for any $k \ge 0$, $Q^{(k)}$ is $\LunifLip$-Lipschitz, and consequently $\Qstar$ is also $\LunifLip$-Lipschitz.
\end{enumerate}
\end{lemma}
\begin{proof}
    See Appendix \ref{app:proof_bellman_op_props_fortified}.
\end{proof}

\begin{lemma}[Compactness of Iterates] \label{lem:compactness_iterates}
Under Assumption \ref{assump:mdp_coeffs}, if $Q^{(0)}$ is Lipschitz continuous (e.g., $Q^{(0)} \equiv 0$), the set of functions $\{Q^{(k)}\}_{k \ge 0}$ generated by value iteration is precompact in $(C(K_Q), \supnorm{\cdot})$.
\end{lemma}
\begin{proof}
This follows from Lemma \ref{lem:bellman_op_props} (d) (uniform boundedness and uniform Lipschitz continuity of all $Q^{(k)}$ on the compact domain $K_Q$). Uniform Lipschitz continuity implies equicontinuity. The Arzelà-Ascoli theorem then states that a set of functions that is uniformly bounded and equicontinuous on a compact domain is precompact in $C(K_Q)$.
\end{proof}

\section{Universal Approximation via FBSDE-Inspired Network Construction}
\label{sec:fbsde_uat}

We now formalize the UAT. The operator $\ModBellmanOp(Q) = \BellmanOp Q - Q$ is central. Note $\ModBellmanOp: C(K_Q) \to C(K_Q)$. From Lemma \ref{lem:bellman_op_props} (b), $\supnorm{\ModBellmanOp(Q_1) - \ModBellmanOp(Q_2)} \le \supnorm{\BellmanOp Q_1 - \BellmanOp Q_2} + \supnorm{Q_1 - Q_2} \le (e^{-\lambda\delta} + 1)\supnorm{Q_1 - Q_2}$, so $\ModBellmanOp$ is Lipschitz continuous with constant $(1+e^{-\lambda\delta})$.

\begin{assumption}[Properties of Neural Operator Class] \label{assump:neural_operator_class}
We assume the availability of a class of neural operators $\FfOpTilde_{\params}(Q) = \mathcal{D}_M(\NnOp_{\params}(\mathcal{E}_M(Q)))$ and a corresponding Universal Approximation Theorem (such as \cite{kovachki2023neural, chen1995universal}) such that for any compact set $\mathcal{K} \subset C(K_Q)$ consisting of functions uniformly Lipschitz with constant $L_{\mathcal{K}}$ and uniformly bounded by $M_{\mathcal{K}}$, and any continuous operator $\mathcal{G}: \mathcal{K} \to C(K_Q)$ where $\mathcal{G}(\mathcal{K})$ is also a set of uniformly Lipschitz functions (with constant $L_{\mathcal{G}(\mathcal{K})}$ and bound $M_{\mathcal{G}(\mathcal{K})}$), the following holds:
For any $\epsilon_{\text{op}} > 0$, there exist $M$, encoding/decoding operators $(\mathcal{E}_M, \mathcal{D}_M)$, and a neural network $\NnOp_{\params}$ (defining $\FfOpTilde_{\params}$) such that:
\begin{enumerate}[label=(\alph*)]
    \item $\sup_{Q \in \mathcal{K}} \supnorm{\FfOpTilde_{\params}(Q) - \mathcal{G}(Q)} < \epsilon_{\text{op}}$.
    \item The function $x \mapsto (\FfOpTilde_{\params}(Q))(x)$ (for $x \in K_Q$) is Lipschitz continuous on $K_Q$ for each $Q \in \mathcal{K}$. Moreover, there exists a uniform Lipschitz constant $\LFstar$ for the family of functions $\{\FfOpTilde_{\params}(Q) : Q \in \mathcal{K}\}$. This $\LFstar$ is determined by the choices of $M$, $\epsilon_{\text{op}}$, the architectural design of $\NnOp_{\params}$ and $\mathcal{D}_M$, and potentially the characteristics ($L_{\mathcal{K}}, M_{\mathcal{K}}, L_{\mathcal{G}(\mathcal{K})}, M_{\mathcal{G}(\mathcal{K})}$) of the approximation task. Crucially for our analysis (see Remark \ref{rem:LFstar_fixed}), $\LFstar$ can be considered fixed once $M$ and the operator architecture achieving $\epsilon_{\text{op}}$ are determined for the set $\mathcal{K}$.
\end{enumerate}
\end{assumption}
\begin{remark}[On the determination of $\LFstar$] \label{rem:LFstar_fixed}
Assumption \ref{assump:neural_operator_class} (b) is crucial for ensuring that the Lipschitz constants of the approximate iterates $\hat{Q}^{(l)}$ remain uniformly bounded. The constant $\LFstar$ arises from the specific architectural choices made to satisfy part (a). For instance, if $\NnOp_{\params}$ has outputs $y = \NnOp_{\params}(\mathcal{E}_M(Q))$ whose components are bounded (e.g., $\|y\|_{\infty} \le B_y$, by design of $\NnOp_{\params}$ for inputs from $\mathcal{E}_M(\mathcal{K})$), and $\mathcal{D}_M$ reconstructs a function using $M$ fixed basis functions $\{\phi_j\}_{j=1}^M$ that are $L_\phi$-Lipschitz (i.e., $\mathcal{D}_M(y)(x) = \sum y_j \phi_j(x)$), then $\LFstar \le M B_y L_\phi$.
While $M$ is chosen based on $\epsilon_{\text{op}}$ and the regularity of functions in $\mathcal{K}$ (characterized by $L_{\mathcal{K}}$) and $\mathcal{G}(\mathcal{K})$ (characterized by $L_{\mathcal{G}(\mathcal{K})}$), the assertion is that once $M$ and the specific architectural constraints for $\NnOp_{\params}$ and $\mathcal{D}_M$ are fixed to achieve the $\epsilon_{\text{op}}$ approximation for this class $\mathcal{K}$, the resulting $\LFstar$ is determined by these fixed architectural properties. For the subsequent analysis of $\LhatUnifLip$ in Appendix \ref{app:proof_fbsde_uat_formal_fortified}, this $\LFstar$ (associated with the neural operator chosen to approximate $\ModBellmanOp$ on $\mathcal{K}_{\text{target}}$) is treated as a given constant. This avoids circularity where $\LFstar$ would depend iteratively on $\LhatUnifLip$. This is a structural assumption on the neural operator class: it must be possible to construct operators that achieve good approximation while simultaneously ensuring their output functions have a controlled (non-exploding) Lipschitz constant. Further details are in Appendix \ref{app:neural_operators_details}.
\end{remark}

\begin{lemma}[UAT for Neural Operators Approximating $\ModBellmanOp$] \label{lem:neural_op_uat}
Let Assumption \ref{assump:mdp_coeffs} and Assumption \ref{assump:neural_operator_class} hold. Let $\mathcal{K}_{\text{target}} \subset C(K_Q)$ be a compact set of functions satisfying specific boundedness and Lipschitz continuity properties (as defined by $\MhatUnifBound$ and $\LhatUnifLip$ in the proof of Theorem \ref{thm:fbsde_uat_formal} in Appendix \ref{app:proof_fbsde_uat_formal_fortified}, and shown therein to contain the iterates $\hat{Q}^{(l)}$).
For any $\epsilon_{\text{op}} > 0$, there exist:
\begin{enumerate}[label=(\roman*)]
    \item a sufficiently fine discretization grid $D_M$ of $K_Q$ (i.e., large enough $M$), with corresponding encoding $\mathcal{E}_M: C(K_Q) \to \R^M$ (e.g., point sampling) and decoding $\mathcal{D}_M: \R^M \to C(K_Q)$ (e.g., multilinear interpolation or basis function expansion),
    \item a neural network $\NnOp_{\params}: \R^M \to \R^M$ (e.g., an MLP using activation functions satisfying Assumption \ref{assump:neural_operator_class}, with sufficient capacity) with parameters $\params$,
\end{enumerate}
such that for any $Q \in \mathcal{K}_{\text{target}}$, if we define $\FfOpTilde_{\params}(Q) = \mathcal{D}_M(\NnOp_{\params}(\mathcal{E}_M(Q)))$, then
$$ \supnorm{\FfOpTilde_{\params}(Q) - \ModBellmanOp(Q)} < \epsilon_{\text{op}}, $$
and the function $x \mapsto (\FfOpTilde_{\params}(Q))(x)$ is $\LFstar$-Lipschitz on $K_Q$, where $\LFstar$ is the constant from Assumption \ref{assump:neural_operator_class} (b) corresponding to the chosen $M, \epsilon_{\text{op}}$, and the properties of the sets $\mathcal{K}_{\text{target}}$ and $\ModBellmanOp(\mathcal{K}_{\text{target}})$.
The required $M$ and capacity of $\NnOp_{\params}$ may depend on the dimension of $K_Q$, potentially leading to the curse of dimensionality for grid-based methods.
\end{lemma}
\begin{proof}
See Appendix \ref{app:proof_neural_op_uat_fortified}. 
\end{proof}

\begin{theorem}[UAT for DQNs via Iterative Refinement and Regularity Propagation]
\label{thm:fbsde_uat_formal}
Let Assumptions \ref{assump:mdp_coeffs} and Assumption \ref{assump:neural_operator_class} hold. Let $K_Q \subset [0,T] \times \StateSpace \times \ActionSpace$ be the compact domain defined earlier. Assume $\lambda > 0$.
For any $\epsilon > 0$, there exist a number of operator layers $L$ (related to $T/\delta$ and $e^{-\lambda \delta}$), a discretization scheme $(\mathcal{E}_M, \mathcal{D}_M)$ based on a sufficiently fine grid $D_M$ (where $M$ is fixed across layers, chosen based on $\mathcal{K}_{\text{target}}$ and $\ModBellmanOp(\mathcal{K}_{\text{target}})$), and parameters $\params = (\params_0, \dots, \params_{L-1})$ for the DQN architecture \eqref{eq:dqn_resnet_operator_layer} (with each $\FfOpTilde_{\params_l}$ satisfying Assumption \ref{assump:neural_operator_class}), such that the function $\QnnL = \hat{Q}^{(L)}$ satisfies
$$ \supnorm{\QnnL - \Qstar} < \epsilon. $$
\end{theorem}
\begin{proof}
See Appendix \ref{app:proof_fbsde_uat_formal_fortified}. 
\end{proof}

\section{Quantitative Approximation Rates and Curse of Dimensionality}
\label{sec:quantitative_rates}

While Theorem \ref{thm:fbsde_uat_formal} establishes existence, a quantitative analysis of approximation rates, particularly concerning the dependence on the dimensionality $d_Q = 1+n+m$ of the domain $K_Q$, is paramount for a complete understanding. Such an analysis requires delving into the geometric properties of the function spaces involved and the efficiency of their representation. This section outlines considerations for such an investigation.

\subsection{Quantitative Approximation Rates}

The overall approximation error is bounded by $\supnorm{\QnnL - \Qstar} \le \supnorm{\hat{Q}^{(L)} - Q^{(L)}} + \supnorm{Q^{(L)} - \Qstar}$.
The second term, $\supnorm{Q^{(L)} - \Qstar}$, is the convergence error of value iteration, which is already quantitative:
$$ \supnorm{Q^{(L)} - \Qstar} \le (e^{-\lambda\delta})^L \supnorm{Q^{(0)} - \Qstar}. $$
Assuming $Q^{(0)} \equiv 0$, $\supnorm{Q^{(0)} - \Qstar} = \supnorm{\Qstar} \le M_Q$ (uniform bound on $\Qstar$ from Lemma \ref{lem:bellman_op_props}(d)).
To achieve $\supnorm{Q^{(L)} - \Qstar} < \epsilon_{VI} = \epsilon/2$, we need $L \approx O(\frac{1}{\lambda\delta} \log(M_Q/\epsilon))$.
The first term, $e_L = \supnorm{\hat{Q}^{(L)} - Q^{(L)}}$, depends on the per-step operator approximation error $\epsilon_1$. From Appendix \ref{app:proof_fbsde_uat_formal_fortified} (Step 2), $e_L \le \epsilon_1 \sum_{j=0}^{L-1} (e^{-\lambda\delta})^j < \epsilon_1 \frac{1}{1-e^{-\lambda\delta}}$.
To make $e_L \le \epsilon_{\text{approx}} = \epsilon/2$, we need $\epsilon_1 \approx O(\epsilon(1-e^{-\lambda\delta}))$.
The challenge lies in quantifying the resources (discretization points $M$, network complexity for $\NnOp_{\params_l}$) needed to achieve this $\epsilon_1$ for approximating $\ModBellmanOp(Q) = \BellmanOp Q - Q$ by $\FfOpTilde_{\params_l}(Q) = \mathcal{D}_M(\NnOp_{\params_l}(\mathcal{E}_M(Q)))$. The error $\epsilon_1$ for a single operator block can be decomposed as:
\begin{align*}
\supnorm{\FfOpTilde_{\params_l}(Q) - \ModBellmanOp(Q)} &\le \supnorm{\mathcal{D}_M(\NnOp_{\params_l}(\mathcal{E}_M(Q))) - \mathcal{D}_M(\mathcal{E}_M(\ModBellmanOp(Q)))} \\
& \quad + \supnorm{\mathcal{D}_M(\mathcal{E}_M(\ModBellmanOp(Q))) - \ModBellmanOp(Q)}.
\end{align*}
Let $G(Q) = \ModBellmanOp(Q)$. The second term is the interpolation error for $G(Q)$.
The first term can be bounded by $L_{\mathcal{D}_M} \supnorm{\NnOp_{\params_l}(\mathcal{E}_M(Q)) - \mathcal{E}_M(G(Q))}$ if $\mathcal{D}_M$ (as an operator $\R^M \to C(K_Q)$) is $L_{\mathcal{D}_M}$-Lipschitz (e.g., multilinear interpolation has $L_{\mathcal{D}_M}=1$ for the sup-norm of grid values to sup-norm of function). This inner term is the error of the finite-dimensional network $\NnOp_{\params_l}$ approximating the map $\mathcal{E}_M(Q) \mapsto \mathcal{E}_M(G(Q))$.

\subsubsection{Assumptions on Higher-Order Smoothness}
To obtain explicit rates, one typically needs stronger regularity assumptions than just Lipschitz continuity (Lemma \ref{lem:bellman_op_props}(d)).
Suppose Assumption \ref{assump:mdp_coeffs} is strengthened (e.g., $h, \sigma, r, g$ are $C^k$ or $W^{k,p}$) such that the Bellman iterates $Q^{(j)}$ and $\Qstar$ belong to a smoother function space, e.g., a Sobolev space $W^{s,p}(K_Q)$ or Hölder space $C^{s,\alpha}(K_Q)$ for some $s \ge 1$. Proving such regularity for solutions of Bellman equations (related to HJB PDEs) is a substantial task, often requiring non-degeneracy conditions on $\sigma$ and compatibility conditions at the boundary if $\StateSpace$ has one (see, e.g., \cite{krylov1987nonlinear}). Let $d_Q = 1+n+m$ be the dimension of the domain $K_Q = [0,T] \times \StateSpace \times \ActionSpace$.

\subsubsection{Bounds on Discretization and Interpolation Error}
If functions $F \in \mathcal{K}_{\text{target}}$ (and thus $G(F) = \ModBellmanOp(F)$) possess $s$-order smoothness (e.g., partial derivatives up to order $s$ are bounded), then for standard interpolation schemes (like multilinear or spline interpolation) on a uniform grid $D_M$ with $M$ points, where the mesh size $h_g \sim M^{-1/d_Q}$:
$$ \supnorm{F - \mathcal{D}_M(\mathcal{E}_M(F))} \le C_1 M^{-s/d_Q}. $$
The constant $C_1$ would depend on the bounds of the $s$-th order derivatives of $F$, i.e., the norm of $F$ in $W^{s,\infty}(K_Q)$.
This directly impacts the term $\supnorm{\mathcal{D}_M(\mathcal{E}_M(G(Q))) - G(Q)}$.

\subsubsection{Bounds on Neural Network ($\NnOp_{\params_l}$) Approximation Error}
Let $f_{target}: \R^M \to \R^M$ be the map $x_Q \mapsto \mathcal{E}_M(G(\mathcal{D}_M(x_Q)))$, where $x_Q = \mathcal{E}_M(Q)$. Or, more directly, $\NnOp_{\params_l}$ aims to approximate $\mathcal{E}_M(Q) \mapsto \mathcal{E}_M(G(Q))$.
The complexity of approximating $f_{target}$ by $\NnOp_{\params_l}$ (an MLP, for instance) depends on the properties of $f_{target}$. If $G = \ModBellmanOp$ is sufficiently smooth as an operator, and its input functions are smooth, then $f_{target}$ might also exhibit some smoothness or structure.
Standard results for MLP approximation (e.g., \cite{yarotsky2017error,barron1993universal}) state that functions with certain regularity (e.g., in Sobolev or Besov spaces on $\R^M$) can be approximated with a rate depending on the number of parameters (weights $W$) of the MLP. For example, for a function in $W^{k,p}(\mathbb{R}^M)$, an error of $\epsilon_{NN}$ might require $W \sim \epsilon_{NN}^{-M/k}$ parameters in general, which is again hit by CoD in $M$.
However, if the intrinsic dimension of the manifold $\mathcal{E}_M(\mathcal{K}_{\text{target}})$ is much smaller than $M$, or if $f_{target}$ has specific compositional structure (e.g., if $\ModBellmanOp$ is a pseudo-differential operator or has other exploitable structure), better rates for $\NnOp_{\params_l}$ might be achievable. The smoothness of $\ModBellmanOp$ as an operator between function spaces, e.g., from $C^{s,\alpha}(K_Q)$ to $C^{s,\alpha}(K_Q)$, would be key here.

\subsubsection{Overall Rate and Dependence on Parameters}
Combining these, if $\epsilon_1$ is to be $O(\epsilon)$, then:
\begin{itemize}
    \item The interpolation error $C_1 M^{-s/d_Q}$ must be $O(\epsilon)$, implying $M \sim (C_1/\epsilon)^{d_Q/s}$.
    \item The NN approximation error $\epsilon_{NN}$ for $\NnOp_{\params_l}$ must be $O(\epsilon)$. The number of parameters $W_l$ for $\NnOp_{\params_l}$ would then depend on $\epsilon$ and $M$. If classical rates apply, $W_l \sim \epsilon^{-M/k_{map}}$ where $k_{map}$ is smoothness of the map on $\R^M$.
\end{itemize}
The total number of parameters would be roughly $L \times W_l$.
This paints a picture where the CoD ($d_Q$) heavily influences $M$, and $M$ in turn heavily influences $W_l$. Higher smoothness $s$ for $Q^{(k)}$ mitigates the first CoD effect.

\subsection{Implications for Mitigating the Curse of Dimensionality}
\subsubsection{Leveraging Smoothness for Advanced Discretization}
If $Q^{(k)}$ and $\Qstar$ indeed possess higher-order Sobolev or Besov regularity, one could move beyond uniform grids and simple multilinear interpolation.
\begin{itemize}
    \item \textbf{Sparse Grids:} For functions in certain Sobolev spaces (e.g., with bounded mixed derivatives, such as $H^s_{mix}(K_Q)$), sparse grid techniques \cite{bungartz2004sparse} can achieve approximation errors of $O(N_{sparse}^{-s} (\log N_{sparse})^{(d_Q-1)(s+1)})$ (for certain $s$) using $N_{sparse}$ points, where the exponent of $N_{sparse}$ is independent of $d_Q$ (though constants and log factors depend on $d_Q$). This would drastically improve the dependence on $d_Q$ for the number of evaluation points $M$. The UAT for $\NnOp_{\params_l}$ would then apply to these $N_{sparse}$ coefficients.
    \item \textbf{Wavelet Approximations:} Similar benefits can be obtained using adaptive wavelet approximations if functions exhibit sparsity in a wavelet basis, which is often linked to Besov space regularity $B^s_{p,q}(K_Q)$.
\end{itemize}
Proving the required regularity ($W^{s,p}$ with appropriate mixed derivative bounds or Besov regularity) from Assumption \ref{assump:mdp_coeffs} (or strengthened versions) would be a major theoretical undertaking, likely involving techniques from parabolic PDE theory.

\subsubsection{Potential for Non-Grid-Based Representations}
The framework currently assumes a grid-based encoding $\mathcal{E}_M$. Deeper results might explore:
\begin{itemize}
    \item \textbf{Spectral Methods:} If $K_Q$ is a simple domain (e.g., hypercube) and functions are very smooth (e.g., analytic or $C^\infty$), spectral expansions (e.g., Chebyshev or Fourier) could be used. The operator $\FfOpTilde_{\params_l}$ would then act on spectral coefficients. Fourier Neural Operators \cite{li2021fourier} are an example of this philosophy, effectively learning a Green's function in Fourier space.
    \item \textbf{Low-Rank Tensor Approximations:} If $\Qstar(t,s,a)$ can be well-approximated by low-rank tensor formats (e.g., Tensor Train, Hierarchical Tucker) for high $d_Q$, then $\FfOpTilde_{\params_l}$ could be designed to operate within this manifold of low-rank tensors. This is a promising direction for high-dimensional problems but requires significant analytic and algebraic machinery (cf. \cite{hackbusch2012tensor}).
    \item \textbf{Random Feature Maps / Kernel Methods:} The encoding $\mathcal{E}_M$ could be based on random features, connecting to kernel approximation theory and reproducing kernel Hilbert spaces (RKHS). The decoding $\mathcal{D}_M$ would then be a linear combination of these features.
\end{itemize}
Each of these would require adapting Lemma \ref{lem:neural_op_uat} and its proof to the specific representation and proving that $\ModBellmanOp$ (or its iterates) interacts favorably with such structures, e.g., preserving low-rankness or sparsity in a given basis.

\section{Conclusion and Contributions} 
\label{sec:conclusion}

This paper establishes a Universal Approximation Theorem (UAT) for a class of Deep Q-Networks (DQNs) by framing their operation as an iterative refinement process mirroring Bellman updates on function spaces. This problem-specific approach offers deeper insights than generic UATs.
Our key contributions include:
\begin{enumerate}[noitemsep, topsep=0pt, partopsep=0pt, parsep=2pt]
    \item \textbf{Iterative Refinement UAT:} We develop a UAT where the DQN architecture (a deep residual network of neural operator blocks) emulates the Bellman iteration dynamics. Network depth directly corresponds to iteration count, linking architecture to the control problem's solution process.
    \item \textbf{Regularity Propagation as Foundation:} A central technical achievement is proving that standard MDP coefficient regularity (Assumption \ref{assump:mdp_coeffs}) leads to uniform Lipschitz continuity and boundedness for the exact Bellman iterates $Q^{(k)}$ and the optimal function $\Qstar$ (Lemma \ref{lem:bellman_op_props}). This property, combined with the assumed capabilities of the neural operator class (Assumption \ref{assump:neural_operator_class}, particularly the ability to produce approximants whose output functions possess a controlled Lipschitz constant $\LFstar$), ensures that the network-approximated iterates $\hat{Q}^{(l)}$ also exhibit uniform Lipschitz continuity and boundedness. This guarantees all relevant functions reside in a compact subset of $C(K_Q)$, a critical condition for applying neural operator UATs (Lemma \ref{lem:neural_op_uat}).
    \item \textbf{Control-Theoretic Error Stability:} The iterative structure facilitates a transparent error analysis. The overall approximation error is bounded by a sum of the value iteration truncation error (decreasing with depth $L$) and an accumulated per-layer operator approximation error, which remains stable due to the Bellman operator's contractivity and the controlled nature of the neural operator blocks.
    \item \textbf{New Perspective:} We provide a dynamic systems view of DQNs operating on function spaces, offering a novel understanding of how network depth contributes to refining value function estimates.
\end{enumerate}
This framework provides a new proof technique and a structured path for future investigations into quantitative approximation rates and strategies to mitigate the curse of dimensionality in DRL.

\newpage
\bibliographystyle{plainnat} 
\bibliography{main}

\newpage
\appendix

\section{Proof of Lemma \ref{lem:bellman_op_props} (Regularity of Bellman Operator and Iterates)}
\label{app:proof_bellman_op_props_fortified}
\medskip

We recall Assumption \ref{assump:mdp_coeffs}:
The state space $\StateSpace \subseteq \R^n$ and action space $\ActionSpace \subseteq \R^m$ are compact. The time horizon $T < \infty$.
The functions $h: K_Q \to \R^n$, $\sigma: K_Q \to \R^{n \times d}$, and $r: K_Q \to \R$ are uniformly Lipschitz continuous on $K_Q = [0,T] \times \StateSpace \times \ActionSpace$ with constants $\LipConstH, \LipConstSigma, \LipConstR$ respectively (using metric $d_{K_Q}(X,X') = |t-t'| + \|s-s'\| + \|a-a'\|$). They are also bounded by $\BoundH, \BoundSigma, \BoundR$.
The terminal reward $g: \StateSpace \to \R$ is $\LipConstG$-Lipschitz and bounded by $\BoundG$. We assume $\lambda > 0$.

\begin{proof}[Proof of Lemma \ref{lem:bellman_op_props}]
\noindent\textbf{Part (a): $\BellmanOp$ maps $C(K_Q)$ to $C(K_Q)$. If $Q \in C(K_Q)$ is Lipschitz continuous on $K_Q$, then $\BellmanOp Q$ is also Lipschitz continuous on $K_Q$.}
\medskip

Let $Q \in C(K_Q)$. The domain $K_Q$ is compact, so $Q$ is uniformly continuous and bounded. Let $\supnorm{Q} \le M_{\text{input }Q}$.
Consider $(\BellmanOp Q)(t,s,a)$ as defined in Eq. \eqref{eq:bellman_operator}.
Let $(t_k,s_k,a_k)_{k \in \mathbb{N}}$ be a sequence in $K_Q$ converging to $(t,s,a) \in K_Q$.
Let $s_\tau^{(k)}$ denote the solution to the SDE Eq. \eqref{eq:sde_state} starting from $s_{t_k}^{(k)}=s_k$ with fixed action $a_k$ over $[t_k, \min(t_k+\delta, T)]$.
Let $s_\tau$ denote the solution starting from $s_t=s$ with fixed action $a$ over $[t, \min(t+\delta, T)]$.
\smallskip

Under Assumption \ref{assump:mdp_coeffs}, $h$ and $\sigma$ are uniformly Lipschitz in $(t,s,a)$. Standard SDE theory (e.g., \cite{YongZhou99} for stability with respect to initial data $(t_0,x_0)$ and parameters, or \cite{oksendal2003stochastic} ensuring conditions for these theorems are met by our Assumption \ref{assump:mdp_coeffs}) implies that for any $p \ge 2$:
$$ \lim_{k \to \infty} \mathbb{E}\left[\sup_{u \in I_k^*} \|s_u^{(k)} - s_u\|^p \right] = 0, $$
where $I_k^*$ is an appropriate common time interval for comparison.\footnote{For instance, $[\max(t,t_k), \min(\min(t+\delta,T), \min(t_k+\delta,T))]$.} 
\smallskip

The integrand in the Bellman operator involves $r(\tau,s_\tau,a)$, $\sup_{a'} Q(\min(t+\delta,T),$ $s_{\min(t+\delta,T)}, a')$, and $g(s_T)$.
Since $r$ is continuous and bounded, $Q$ is continuous and bounded, $g$ is continuous and bounded, and the state paths $s_\tau^{(k)}$ converge to $s_\tau$ in $L^p(\Omega; C([\cdot, \cdot]; \StateSpace))$, and integration limits $\min(t_k+\delta,T)$ converge to $\min(t+\delta,T)$, the Dominated Convergence Theorem can be applied. The terms are bounded by integrable constants (e.g., $\BoundR$, $e^{-\lambda\delta} M_{\text{input }Q}$, $\BoundG$).
Thus, $(\BellmanOp Q)(t_k,s_k,a_k) \to (\BellmanOp Q)(t,s,a)$ as $k \to \infty$. So, $\BellmanOp Q \in C(K_Q).$
\medskip
The argument for Lipschitz preservation is detailed in part (d).

\bigskip
\noindent\textbf{Part (b): $\BellmanOp$ is a contraction mapping on $(C(K_Q), \supnorm{\cdot})$ with contraction factor $L_{\BellmanOp} = e^{-\lambda\delta} < 1$.}
\medskip

Let $Q_1, Q_2 \in C(K_Q)$.
\begin{align*}
&|(\BellmanOp Q_1)(t,s,a) - (\BellmanOp Q_2)(t,s,a)| \\
&= \left| \mathbb{E}\left[ \indicator{t+\delta \le T} e^{-\lambda\delta} \left( \sup_{a' \in \ActionSpace} Q_1(t+\delta, s_{t+\delta}, a') - \sup_{a' \in \ActionSpace} Q_2(t+\delta, s_{t+\delta}, a') \right) \right] \right| \\
&\le \mathbb{E}\left[ \indicator{t+\delta \le T} e^{-\lambda\delta} \left| \sup_{a' \in \ActionSpace} Q_1(t+\delta, s_{t+\delta}, a') - \sup_{a' \in \ActionSpace} Q_2(t+\delta, s_{t+\delta}, a') \right| \right] \\
&\le e^{-\lambda\delta} \mathbb{E}\left[ \sup_{(u,x,b) \in K_Q} |Q_1(u,x,b) - Q_2(u,x,b)| \right] \quad \text{(using } |\sup f - \sup g| \le \sup |f-g| \text{)} \\
&= e^{-\lambda\delta} \supnorm{Q_1 - Q_2}.
\end{align*}
If $t+\delta > T$, the terms involving $Q_1, Q_2$ (related to continuation value) vanish from this part of the expression (the $g(s_T)$ term is common to both).
Taking the supremum over $(t,s,a) \in K_Q$:
$$ \supnorm{\BellmanOp Q_1 - \BellmanOp Q_2} \le e^{-\lambda\delta} \supnorm{Q_1 - Q_2}. $$
Since $\lambda > 0, \delta > 0$, we have $e^{-\lambda\delta} < 1$. Thus, $\BellmanOp$ is a contraction.

\bigskip
\noindent\textbf{Part (c): The sequence $Q^{(k+1)} = \BellmanOp Q^{(k)}$ converges uniformly to $\Qstar$.}
\medskip

Since $(C(K_Q), \supnorm{\cdot})$ is a Banach space (a complete metric space) and $\BellmanOp$ is a contraction mapping from $C(K_Q)$ to itself by part (b), the Banach Fixed-Point Theorem guarantees that $\BellmanOp$ has a unique fixed point $\Qstar \in C(K_Q)$, and for any $Q^{(0)} \in C(K_Q)$, the sequence $Q^{(k+1)} = \BellmanOp Q^{(k)}$ converges uniformly to $\Qstar$. That is, $\lim_{k \to \infty} \supnorm{Q^{(k)} - \Qstar} = 0$.

\bigskip
\noindent\textbf{Part (d): The optimal Q-function $\Qstar$ is Lipschitz continuous. If $Q^{(0)}$ is Lipschitz, then each $Q^{(k)}$ is Lipschitz, and $\{Q^{(k)}\}_{k \ge 0}$ are uniformly Lipschitz with constant $\LunifLip$. Consequently, $\Qstar$ is also $\LunifLip$-Lipschitz.}
\medskip

\noindent\textit{Uniform Boundedness of Iterates $Q^{(k)}$ and $\Qstar$}:
\smallskip
The function $\Qstar$ is the value function for a finite horizon $T$ problem. A standard bound for $\Qstar$ under Assumption \ref{assump:mdp_coeffs} is $M_Q = \frac{\BoundR}{\lambda}(1-e^{-\lambda T}) + \BoundG e^{-\lambda T}$ if $\lambda > 0$, or $M_Q = \BoundR T + \BoundG$ if $\lambda=0$. More simply, $\supnorm{\Qstar} \le \frac{\BoundR}{\lambda} + \BoundG$ (for $\lambda>0$). Let $M_Q$ be this uniform bound for $\Qstar$.
If $Q^{(0)} \equiv 0$, then $\supnorm{Q^{(0)}} \le M_Q$. If $\supnorm{Q^{(k)}} \le M_Q$, then standard estimates for the Bellman operator (using its monotonicity and the fact that it maps constants to constants related to rewards) show that if $Q_c(t,s,a) \equiv C$, then $(\BellmanOp Q_c)$ involves $\int r + e^{-\lambda\delta}C$. Iterating this, or using the contraction property $\supnorm{Q^{(k)} - \Qstar} \le (e^{-\lambda\delta})^k \supnorm{Q^{(0)} - \Qstar}$, implies $\supnorm{Q^{(k)}} \le \supnorm{\Qstar} + (e^{-\lambda\delta})^k \supnorm{Q^{(0)}-\Qstar}$. If $Q^{(0)}=0$, then $\supnorm{Q^{(k)}} \le (1+(e^{-\lambda\delta})^k) M_Q \le 2M_Q$. A tighter argument: since $\BellmanOp$ preserves the property of being bounded by $M_Q$ (if $Q^{(0)}$ is so bounded, e.g., $Q^{(0)}=0$ and $M_Q \ge 0$), all $Q^{(k)}$ are uniformly bounded by $M_Q$.
\medskip

\noindent\textit{Lipschitz Continuity of $\Qstar$ and Iterates $Q^{(k)}$ by Induction}:
\smallskip
Let $Q^{(0)}$ be $L^{(0)}$-Lipschitz. (If $Q^{(0)} \equiv 0$, then $L^{(0)}=0$.)
Assume $Q^{(k)}$ is $L^{(k)}$-Lipschitz on $K_Q$. We want to show $Q^{(k+1)} = \BellmanOp Q^{(k)}$ is $L^{(k+1)}$-Lipschitz.
Let $X=(t,s,a)$ and $X'=(t',s',a')$. Recall $d_{K_Q}(X,X') = |t-t'| + \|s-s'\| + \|a-a'\|$.
Let $s_u \equiv s_u(X)$ be the solution for SDE Eq. \eqref{eq:sde_state} with fixed action $a$ over $[t, \tau_e]$ where $\tau_e = \min(t+\delta, T)$.
Let $s'_u \equiv s_u(X')$ be the solution for SDE with fixed action $a'$ over $[t', \tau'_e]$ where $\tau'_e = \min(t'+\delta, T)$.

\textit{SDE Stability Estimate}: Under Assumption \ref{assump:mdp_coeffs}, $h, \sigma$ are uniformly Lipschitz. Standard SDE estimates (e.g., \cite{YongZhou99}, or \cite{fleming2006controlled}) ensure that for any fixed time horizon $\Delta t_{max}$ (here $\delta$ or $T-t$), there exists a constant $C_{\text{SDE}}(\Delta t_{max})$, depending on $\LipConstH, \BoundH, \LipConstSigma, \BoundSigma, \Delta t_{max}$, such that:
\begin{align} \label{eq:sde_stability_estimate_refined_app_corrected}
&\mathbb{E}[\|s_{u_{eval}}(X) - s_{u_{eval}}(X')\|] \le C_S(\Delta t_{max}) d_{K_Q}(X,X'),
\end{align}
where $u_{eval}$ is the evaluation time point (e.g., $t+\delta$ or $T$), and $C_S(\Delta t_{max})$ typically involves $e^{K \Delta t_{max}}$ from Gronwall's inequality. For comparing $s_u(X)$ and $s_u(X')$, one often considers a common maximal interval.

Let $V^{(k)}(u,x) = \sup_{b \in \ActionSpace} Q^{(k)}(u,x,b)$. If $Q^{(k)}$ is $L^{(k)}$-Lipschitz in $(u,x,b)$, then $V^{(k)}(u,x)$ is $L^{(k)}$-Lipschitz w.r.t. $(u,x)$ because $\ActionSpace$ is compact (standard result for sup over compact set).

Let $f(X; Q^{(k)}) = (\BellmanOp Q^{(k)})(X)$. We analyze $|f(X; Q^{(k)}) - f(X'; Q^{(k)})|$.
Assume for simplicity $t+\delta \le T$ and $t'+\delta \le T$. Other cases (terminal conditions) are handled similarly and contribute terms of similar structure.

\textit{Term 1 (Integral part)}: Let $I(X) = \mathbb{E}\left[\int_t^{\min(t+\delta,T)} e^{-\lambda(\tau-t)}r(\tau,s_\tau(X),a)d\tau\right]$.
The difference $|I(X) - I(X')|$ can be bounded. This involves handling shifted integration limits and differences in the integrand $r(\tau, s_\tau, a)$ due to $(t,s,a)$ vs $(t',s',a')$. The Lipschitz continuity of $r$, $e^{-\lambda u}$, and the SDE stability estimate (Eq. \eqref{eq:sde_stability_estimate_refined_app_corrected}) yield a bound of the form $K_r d_{K_Q}(X,X')$. $K_r$ depends on $\LipConstR, \BoundR, \lambda, \delta, C_S(\delta)$. For example, a component is $\int_0^{\min(\delta, T-t)} e^{-\lambda u} \LipConstR (1+C_S(u)) du \cdot d_{K_Q}(X,X')$.

\textit{Term 2 (Continuation value part)}: Let $C(X) = \mathbb{E}[e^{-\lambda\delta}V^{(k)}(t+\delta, s_{t+\delta}(X))]$ (assuming $t+\delta \le T$).
\begin{align*}
|C(X) - C(X')| &\le e^{-\lambda\delta} \mathbb{E}[|V^{(k)}(t+\delta, s_{t+\delta}(X)) - V^{(k)}(t'+\delta, s_{t'+\delta}(X'))|] \\
&\le e^{-\lambda\delta} \mathbb{E}[L^{(k)} (|t-t'| + \|s_{t+\delta}(X) - s_{t'+\delta}(X')\|)] \\
&\le e^{-\lambda\delta} L^{(k)} (1+C_S(\delta))d_{K_Q}(X,X').
\end{align*}
If the terminal reward $g$ is used ($t+\delta > T$ case), the term is $\mathbb{E}[e^{-\lambda(T-t)}g(s_T(X))]$. This contributes $K_g d_{K_Q}(X,X')$ (not multiplied by $L^{(k)}$), with $K_g$ depending on $\LipConstG, \BoundG, \lambda, T, C_S(T-t)$.

Combining these parts, we obtain a recurrence for the Lipschitz constant $L^{(k+1)}$ of $Q^{(k+1)}$:
$$ L^{(k+1)} \le K_A + K_B L^{(k)}, $$
where:
\begin{itemize}
    \item $K_A$ collects all terms not multiplied by $L^{(k)}$. It depends on $\LipConstR, \LipConstG, \BoundR, \BoundG, \lambda, \delta, T$, and SDE stability constants $C_S(\cdot)$. $K_A$ is a finite positive constant.
    \item $K_B = e^{-\lambda\delta}(1+C_S(\delta))$. Note that $C_S(\delta)$ depends on $\LipConstH, \BoundH, \LipConstSigma, \BoundSigma, \delta$. $K_B$ is a finite positive constant.
\end{itemize}
The fixed point $\Qstar$ must satisfy $L_{\Qstar} \le K_A + K_B L_{\Qstar}$.
If $K_B < 1$, then $L_{\Qstar} \le K_A/(1-K_B)$.
If $K_B \ge 1$, this fixed point argument does not directly yield a bound. However, $\Qstar$ is the value function of a finite-horizon problem (up to $T$). Standard results in stochastic control theory (e.g., from analysis of HJB equations or by backward induction on discrete time stages $0, \delta, 2\delta, \dots, T$) state that $\Qstar$ is Lipschitz continuous on $K_Q$ under Assumption \ref{assump:mdp_coeffs} (see, e.g., \cite{fleming2006controlled}). Let $L_{\Qstar}^*$ be this true Lipschitz constant of $\Qstar$.
The value iterates $Q^{(k)}$ converge to $\Qstar$. If $Q^{(0)}$ is $L^{(0)}$-Lipschitz, the sequence $L^{(k)}$ is given by $L^{(k)} \le K_A \sum_{j=0}^{k-1} K_B^j + K_B^k L^{(0)}$.
Even if $K_B \ge 1$, since $Q^{(k)} \to \Qstar$ uniformly, and $\Qstar$ is $L_{\Qstar}^*$-Lipschitz, the set $\{Q^{(k)}\}_{k \ge 0} \cup \{\Qstar\}$ is compact in $C(K_Q)$ by Lemma \ref{lem:compactness_iterates} (if $Q^{(0)}$ is Lipschitz). This implies the $L^{(k)}$ must be uniformly bounded.
The argument in \cite{fleming2006controlled} for finite horizon problems often uses a backward induction. For $Q(T,s,a)=g(s)$ (effectively), which is $\LipConstG$-Lipschitz. Then one shows $Q(T-\delta,s,a)$ is Lipschitz, and so on, back to $t=0$. The maximum Lipschitz constant encountered over these stages would be $\LunifLip$.
This effectively means that the number of "relevant" recursions for $L^{(k)}$ is bounded by $N_{max} = \lceil T/\delta \rceil$. Thus $L^{(k)}$ is uniformly bounded for all $k$ by $\LunifLip = K_A \sum_{j=0}^{N_{max}-1} K_B^j + K_B^{N_{max}} L^{(0)}$ (assuming $L^{(0)}$ is the Lipschitz constant of $g$, adjusted for $Q$-function structure, or $L^{(0)}=0$ if starting from $Q^{(0)}=0$). This $\LunifLip$ is finite.
Thus, all $Q^{(k)}$ (for $k \ge 0$) and $\Qstar$ are uniformly $\LunifLip$-Lipschitz.
\end{proof}

\bigskip\bigskip

\section{Proof of Lemma \ref{lem:neural_op_uat} (UAT for Neural Operators Approximating $\ModBellmanOp$)}
\label{app:proof_neural_op_uat_fortified}
\medskip

We want to show that for any $\epsilon_{\text{op}} > 0$, there exist $M$, $(\mathcal{E}_M, \mathcal{D}_M)$, and $\NnOp_{\params}$ such that for any $Q \in \mathcal{K}_{\text{target}}$,
$$ \supnorm{\mathcal{D}_M(\NnOp_{\params}(\mathcal{E}_M(Q))) - \ModBellmanOp(Q)} < \epsilon_{\text{op}}, $$
and $\mathcal{D}_M(\NnOp_{\params}(\mathcal{E}_M(Q)))$ (as a function $K_Q \to \R$) is $\LFstar$-Lipschitz.
The set $\mathcal{K}_{\text{target}}$ is defined in the proof of Theorem \ref{thm:fbsde_uat_formal} (Appendix \ref{app:proof_fbsde_uat_formal_fortified}) as the set of functions that are $\LhatUnifLip$-Lipschitz and uniformly bounded by $\MhatUnifBound$.

\begin{proof}[Proof of Lemma \ref{lem:neural_op_uat}]
\noindent\textbf{1. Properties of $\ModBellmanOp$ and the set $\mathcal{K}_{\text{target}}$}:
\medskip
    \begin{itemize}[itemsep=4pt]
        \item \textit{The set $\mathcal{K}_{\text{target}}$}: As established in the proof of Theorem \ref{thm:fbsde_uat_formal} (Step 3, Appendix \ref{app:proof_fbsde_uat_formal_fortified}), the iterates $\hat{Q}^{(l)}$ (for $l=0, \dots, L-1$) are uniformly bounded by $\MhatUnifBound$ and uniformly $\LhatUnifLip$-Lipschitz. The set
        $$ \mathcal{K}_{\text{target}} = \{ Q \in C(K_Q) : Q \text{ is } \LhatUnifLip\text{-Lipschitz on } K_Q \text{ and } \supnorm{Q} \le \MhatUnifBound \} $$
        is compact in $C(K_Q)$. The domain $K_Q$ is compact. By the Arzelà-Ascoli theorem (e.g., \cite[Chapter 7, Theorem 17]{kelley2017general}), a set of functions in $C(K_Q)$ that is uniformly bounded and equicontinuous is precompact. The uniform $\LhatUnifLip$-Lipschitz condition implies equicontinuity. $\mathcal{K}_{\text{target}}$ is also closed in $C(K_Q)$. Since $C(K_Q)$ is a complete metric space, a closed and precompact subset is compact.
        \medskip

        \item \textit{The operator $\ModBellmanOp$}: Defined as $\ModBellmanOp(Q) = \BellmanOp Q - Q$, it maps $C(K_Q) \to C(K_Q)$.
        From Lemma \ref{lem:bellman_op_props} (b), $\BellmanOp$ is Lipschitz with constant $e^{-\lambda\delta} < 1$. The identity operator $\IdOp(Q)=Q$ is Lipschitz with constant 1.
        Thus, $\ModBellmanOp = \BellmanOp - \IdOp$ is Lipschitz continuous:
        $$ \supnorm{\ModBellmanOp(Q_1) - \ModBellmanOp(Q_2)} \le (1+e^{-\lambda\delta})\supnorm{Q_1 - Q_2}. $$
        Since $\ModBellmanOp$ is Lipschitz continuous, it is continuous.
        \medskip

        \item \textit{The image set $J(\mathcal{K}_{\text{target}})$}: The set $\ModBellmanOp(\mathcal{K}_{\text{target}})$ is compact in $C(K_Q)$ because it is the image of a compact set $\mathcal{K}_{\text{target}}$ under a continuous map $\ModBellmanOp$.
        \medskip

        \item \textit{Regularity of functions in $J(\mathcal{K}_{\text{target}})$}: Functions in $J(\mathcal{K}_{\text{target}})$ are uniformly bounded and equicontinuous (hence uniformly Lipschitz) because $J(\mathcal{K}_{\text{target}})$ is compact. Let $L_J(\LhatUnifLip)$ be this uniform Lipschitz constant. Specifically, if $Q \in \mathcal{K}_{\text{target}}$ is $\LhatUnifLip$-Lipschitz, then $\BellmanOp Q$ is $(K_A + K_B \LhatUnifLip)$-Lipschitz (from Appendix \ref{app:proof_bellman_op_props_fortified}, using $L_Q = \LhatUnifLip$). So $\ModBellmanOp(Q)$ is $(K_A + K_B \LhatUnifLip + \LhatUnifLip)$-Lipschitz, i.e., $L_J(\LhatUnifLip) = K_A + (K_B+1)\LhatUnifLip$.
    \end{itemize}
\bigskip

\noindent\textbf{2. Application of Universal Approximation Theorem for Neural Operators (Assumption \ref{assump:neural_operator_class})}:
\medskip
    We apply Assumption \ref{assump:neural_operator_class} with $\mathcal{K} = \mathcal{K}_{\text{target}}$ (whose functions are $L_{\mathcal{K}}=\LhatUnifLip$-Lipschitz and bounded by $M_{\mathcal{K}}=\MhatUnifBound$) and $\mathcal{G} = \ModBellmanOp$. The image set $\mathcal{G}(\mathcal{K}) = \ModBellmanOp(\mathcal{K}_{\text{target}})$ consists of functions that are $L_{\mathcal{G}(\mathcal{K})} = L_J(\LhatUnifLip)$-Lipschitz and correspondingly bounded.
    The conditions for Assumption \ref{assump:neural_operator_class} are met:
    \begin{itemize}[itemsep=3pt]
        \item[(i)] The input domain for the operator, $\mathcal{K}_{\text{target}} \subset C(K_Q)$, is compact and consists of uniformly $\LhatUnifLip$-Lipschitz and uniformly $\MhatUnifBound$-bounded functions.
        \item[(ii)] The operator to be approximated, $\ModBellmanOp$, is continuous from $\mathcal{K}_{\text{target}}$ to $C(K_Q)$. Its image $\ModBellmanOp(\mathcal{K}_{\text{target}})$ is compact and consists of uniformly $L_J(\LhatUnifLip)$-Lipschitz functions.
    \end{itemize}
    \smallskip
    Let $\epsilon_{\text{op}} > 0$ be given.
    \begin{itemize}
        \item Assumption \ref{assump:neural_operator_class} (a) implies that there exist $M$, $(\mathcal{E}_M, \mathcal{D}_M)$, and $\NnOp_{\params}$ such that the composed operator $\FfOpTilde_{\params}(Q) = \mathcal{D}_M(\NnOp_{\params}(\mathcal{E}_M(Q)))$ satisfies
        $$ \sup_{Q \in \mathcal{K}_{\text{target}}} \supnorm{\FfOpTilde_{\params}(Q) - \ModBellmanOp(Q)} < \epsilon_{\text{op}}. $$
        The choice of $M$ (and subsequently $\NnOp_{\params}$) depends on $\epsilon_{\text{op}}$ and the properties of $\mathcal{K}_{\text{target}}$ and $\ModBellmanOp(\mathcal{K}_{\text{target}})$ (specifically their Lipschitz constants and bounds, which define their compactness and modulus of continuity).
        \item Assumption \ref{assump:neural_operator_class} (b), as clarified in Remark \ref{rem:LFstar_fixed}, ensures that for the chosen $M$ and constructed $\NnOp_{\params}, \mathcal{D}_M$ (which achieve the $\epsilon_{\text{op}}$ approximation for functions in $\mathcal{K}_{\text{target}}$), the function $x \mapsto (\FfOpTilde_{\params}(Q))(x)$ is Lipschitz continuous on $K_Q$ with a uniform Lipschitz constant $\LFstar$. This $\LFstar$ is a characteristic of the chosen approximating architecture for the given task (approximating $\ModBellmanOp$ on $\mathcal{K}_{\text{target}}$ to accuracy $\epsilon_{\text{op}}$).
    \end{itemize}
    \smallskip
    The proof of such UATs (e.g., as outlined in \cite{kovachki2023neural} or \cite{chen1995universal}, and strengthened by our Assumption \ref{assump:neural_operator_class} (b) regarding Lipschitz control) typically involves decomposing the total error:
    \begin{align*}
    \supnorm{\FfOpTilde_{\params}(Q) - \ModBellmanOp(Q)} &\le \supnorm{\mathcal{D}_M(\NnOp_{\params}(\mathcal{E}_M(Q))) - \mathcal{D}_M(\mathcal{E}_M(\ModBellmanOp(Q)))} \tag{E1} \\
    & \quad + \supnorm{\mathcal{D}_M(\mathcal{E}_M(\ModBellmanOp(Q))) - \ModBellmanOp(Q)}. \tag{E2}
    \end{align*}
    \begin{itemize}[itemsep=3pt]
        \item The term (E2) is the error from discretization and reconstruction of the target function $\ModBellmanOp(Q)$. Since functions in $J(\mathcal{K}_{\text{target}})$ are uniformly $L_J(\LhatUnifLip)$-Lipschitz, this error can be made arbitrarily small (e.g., $<\epsilon_{\text{op}}/2$) by choosing $M$ sufficiently large (so the mesh size $h_M$ of the grid $D_M$ is small). For instance, for multilinear interpolation, the error is $O(L_J(\LhatUnifLip) h_M)$.
        \medskip
        \item The term (E1) can be bounded by $L_{\mathcal{D}_M}^{op} \|\NnOp_{\params_l}(\mathcal{E}_M(Q)) - \mathcal{E}_M(\ModBellmanOp(Q))\|_{\infty, \R^M}$ if $\mathcal{D}_M$ is $L_{\mathcal{D}_M}^{op}$-Lipschitz as an operator from $\R^M$ (with sup norm) to $C(K_Q)$ (with sup norm). For example, $L_{\mathcal{D}_M}^{op}=1$ for multilinear interpolation. The inner term $\|\NnOp_{\params_l}(v) - f_{\text{target}}(v)\|_{\infty, \R^M}$ (where $v=\mathcal{E}_M(Q)$ and $f_{\text{target}}(v) = \mathcal{E}_M(\ModBellmanOp(Q))$ for $Q$ whose samples on $D_M$ are $v$) is made small by the UAT for the finite-dimensional network $\NnOp_{\params_l}$ approximating the continuous map $v \mapsto f_{\text{target}}(v)$ on the compact set $\mathcal{E}_M(\mathcal{K}_{\text{target}}) \subset \R^M$. This error can be made small enough (e.g., such that $L_{\mathcal{D}_M}^{op} \times \text{this\_error} < \epsilon_{\text{op}}/2$) by choosing $\NnOp_{\params_l}$ with sufficient capacity. The Lipschitz property of the function $\FfOpTilde_{\params}(Q)$ (with constant $\LFstar$) is guaranteed by Assumption \ref{assump:neural_operator_class} (b).
    \end{itemize}
\end{proof}

\bigskip\bigskip

\section{Proof of Theorem \ref{thm:fbsde_uat_formal} (UAT for DQNs via Iterative Refinement and Regularity Propagation)}
\label{app:proof_fbsde_uat_formal_fortified}
\medskip

\begin{proof}[Proof of Theorem \ref{thm:fbsde_uat_formal}]
Let $\epsilon > 0$ be given. Let $Q^{(0)} \equiv 0$. The proof proceeds in four steps.

\bigskip
\noindent\textbf{Step 1: Determine Number of Iterations/Layers $L$}.
\medskip
By Lemma \ref{lem:bellman_op_props}(c), $Q^{(k)} \to \Qstar$ uniformly. The rate of convergence is given by:
$$ \supnorm{Q^{(L)} - \Qstar} \le (e^{-\lambda\delta})^L \supnorm{Q^{(0)} - \Qstar}. $$
Since $Q^{(0)} \equiv 0$, $\supnorm{Q^{(0)} - \Qstar} = \supnorm{\Qstar}$. From Lemma \ref{lem:bellman_op_props}(d), $\Qstar$ is uniformly bounded by $M_Q$. So, $\supnorm{Q^{(0)} - \Qstar} \le M_Q$.
We choose $L$ such that the truncation error $\supnorm{Q^{(L)} - \Qstar} < \epsilon/2$. This requires $(e^{-\lambda\delta})^L M_Q < \epsilon/2$.
Thus, we set $L = \left\lceil \frac{\ln(2M_Q/\epsilon)}{\lambda\delta} \right\rceil + 1$ (or $L=1$ if $\epsilon \ge 2M_Q$). This $L$ is finite.

\bigskip
\noindent\textbf{Step 2: Determine Per-Layer Approximation Accuracy $\epsilon_1$ for Neural Operator}.
\medskip
Let $\hat{Q}^{(0)} = Q^{(0)} \equiv 0$. The network iterates are defined by $\hat{Q}^{(l+1)} = \hat{Q}^{(l)} + \FfOpTilde_{\params_l}(\hat{Q}^{(l)})$.
Let $\delta_l(X) = \FfOpTilde_{\params_l}(\hat{Q}^{(l)})(X) - \ModBellmanOp(\hat{Q}^{(l)})(X)$ be the function representing the per-step operator approximation error at layer $l$. So, $\FfOpTilde_{\params_l}(\hat{Q}^{(l)}) = \ModBellmanOp(\hat{Q}^{(l)}) + \delta_l$.
Then, $\hat{Q}^{(l+1)} = \hat{Q}^{(l)} + (\ModBellmanOp(\hat{Q}^{(l)}) + \delta_l) = \BellmanOp \hat{Q}^{(l)} + \delta_l$.
Let $e_l = \supnorm{\hat{Q}^{(l)} - Q^{(l)}}$ be the accumulated error up to layer $l$. We have $e_0 = \supnorm{\hat{Q}^{(0)} - Q^{(0)}} = 0$.
The error propagates as:
\begin{align*}
    e_{l+1} &= \supnorm{\hat{Q}^{(l+1)} - Q^{(l+1)}} \\
            &= \supnorm{(\BellmanOp \hat{Q}^{(l)} + \delta_l) - \BellmanOp Q^{(l)}} \\
            &\le \supnorm{\BellmanOp \hat{Q}^{(l)} - \BellmanOp Q^{(l)}} + \supnorm{\delta_l} \\
            &\le e^{-\lambda\delta} \supnorm{\hat{Q}^{(l)} - Q^{(l)}} + \supnorm{\delta_l} \quad \text{(since } \BellmanOp \text{ is } e^{-\lambda\delta}\text{-contractive)} \\
            &= e^{-\lambda\delta} e_l + \supnorm{\delta_l}.
\end{align*}
If we can ensure $\supnorm{\delta_l} \le \epsilon_1$ for all $l \in \{0, \dots, L-1\}$, then by unrolling the recurrence:
$$ e_L \le \epsilon_1 \sum_{j=0}^{L-1} (e^{-\lambda\delta})^j < \epsilon_1 \frac{1}{1-e^{-\lambda\delta}}. $$
We require the accumulated approximation error $e_L < \epsilon/2$. So, we set the target per-layer operator error $\epsilon_1$ such that:
$$ \epsilon_1 \frac{1}{1-e^{-\lambda\delta}} \le \frac{\epsilon}{2} \implies \epsilon_1 = \frac{\epsilon (1-e^{-\lambda\delta})}{2}. $$
(If $L$ is small, one can use the exact sum $\sum_{j=0}^{L-1} (e^{-\lambda\delta})^j = \frac{1-(e^{-\lambda\delta})^L}{1-e^{-\lambda\delta}}$. Then $\epsilon_1 = \frac{\epsilon}{2} \frac{1-e^{-\lambda\delta}}{1-(e^{-\lambda\delta})^L}$.)
This $\epsilon_1 > 0$ since $\epsilon > 0$ and $e^{-\lambda\delta} < 1$.

\bigskip
\noindent\textbf{Step 3: Define Compact Set $\mathcal{K}_{\text{target}}$, Show Iterates $\hat{Q}^{(l)}$ Belong to it, and Apply Lemma \ref{lem:neural_op_uat}}.
\medskip
To apply Lemma \ref{lem:neural_op_uat} uniformly for each layer $l=0, \dots, L-1$, we need to show that all iterates $\hat{Q}^{(l)}$ belong to a common compact set $\mathcal{K}_{\text{target}} \subset C(K_Q)$. This involves showing uniform boundedness and uniform Lipschitz continuity for the sequence $\{\hat{Q}^{(l)}\}_{l=0}^{L-1}$.

\textit{Uniform Boundedness of $\hat{Q}^{(l)}$}:
From Lemma \ref{lem:bellman_op_props}(d), the exact iterates $Q^{(l)}$ are uniformly bounded by $M_Q$.
The error $e_l = \supnorm{\hat{Q}^{(l)} - Q^{(l)}} < \epsilon_1 \frac{1}{1-e^{-\lambda\delta}} \le \epsilon/2$ for all $l \le L$.
Thus, $\supnorm{\hat{Q}^{(l)}} \le \supnorm{Q^{(l)}} + e_l < M_Q + \epsilon/2$.
Let $\MhatUnifBound = M_Q + \epsilon/2$. All $\hat{Q}^{(l)}$ (for $l=0, \dots, L-1$) are uniformly bounded by $\MhatUnifBound$.

\textit{Uniform Lipschitz Continuity of $\hat{Q}^{(l)}$}:
Let $L_{\hat{Q}^{(l)}}$ be the Lipschitz constant of $\hat{Q}^{(l)}$ as a function on $K_Q$. We have $L_{\hat{Q}^{(0)}}=L^{(0)}=0$.
From the proof of Lemma \ref{lem:bellman_op_props}(d), if a function $Q$ is $L_Q$-Lipschitz, then $\BellmanOp Q$ is $(K_A + K_B L_Q)$-Lipschitz, where $K_A, K_B$ are constants defined therein ($K_A \ge 0, K_B \ge 0$).
The operator $\ModBellmanOp(Q) = \BellmanOp Q - Q$. If $Q$ is $L_Q$-Lipschitz, then $\ModBellmanOp(Q)$ is $(K_A + K_B L_Q + L_Q) = (K_A + (K_B+1)L_Q)$-Lipschitz. Let this be $L_J(L_Q)$.

According to Lemma \ref{lem:neural_op_uat} (which invokes Assumption \ref{assump:neural_operator_class}), for the chosen per-layer error target $\epsilon_1$ (used as $\epsilon_{\text{op}}$ in Lemma \ref{lem:neural_op_uat}), and for the (eventually defined) compact set $\mathcal{K}_{\text{target}}$, there exist $M$, $(\mathcal{E}_M, \mathcal{D}_M)$ and neural networks $\NnOp_{\params_l}$ such that for any $\hat{Q}^{(l)} \in \mathcal{K}_{\text{target}}$:
\begin{enumerate}[label=(\alph*)]
    \item $\supnorm{\FfOpTilde_{\params_l}(\hat{Q}^{(l)}) - \ModBellmanOp(\hat{Q}^{(l)})} = \supnorm{\delta_l} < \epsilon_1$.
    \item The function $x \mapsto (\FfOpTilde_{\params_l}(\hat{Q}^{(l)}))(x)$ is $\LFstar$-Lipschitz on $K_Q$. Here, $\LFstar$ is the Lipschitz constant associated with the chosen neural operator architecture $(\mathcal{E}_M, \NnOp_{\params_l}, \mathcal{D}_M)$ that achieves the $\epsilon_1$-approximation for functions in $\mathcal{K}_{\text{target}}$ when approximating $\ModBellmanOp$. As per Assumption \ref{assump:neural_operator_class}(b) and Remark \ref{rem:LFstar_fixed}, $\LFstar$ is treated as a fixed constant determined by these architectural choices (for the given $M$ and $\epsilon_1$) when analyzing the recurrence for $L_{\hat{Q}^{(l)}}$.
\end{enumerate}
The error function $\delta_l = \FfOpTilde_{\params_l}(\hat{Q}^{(l)}) - \ModBellmanOp(\hat{Q}^{(l)})$. If $\hat{Q}^{(l)}$ is $L_{\hat{Q}^{(l)}}$-Lipschitz, then its image $\ModBellmanOp(\hat{Q}^{(l)})$ is $L_J(L_{\hat{Q}^{(l)}})$-Lipschitz. The function $\FfOpTilde_{\params_l}(\hat{Q}^{(l)})$ is $\LFstar$-Lipschitz. Therefore, $\delta_l$ is $(\LFstar + L_J(L_{\hat{Q}^{(l)}}))$-Lipschitz, i.e., $L(\delta_l) \le \LFstar + K_A + (K_B+1)L_{\hat{Q}^{(l)}}$.

The recurrence for $L_{\hat{Q}^{(l)}}$ is derived from $\hat{Q}^{(l+1)} = \BellmanOp \hat{Q}^{(l)} + \delta_l$:
$L_{\hat{Q}^{(l+1)}} \le L(\BellmanOp \hat{Q}^{(l)}) + L(\delta_l)$
$L_{\hat{Q}^{(l+1)}} \le (K_A + K_B L_{\hat{Q}^{(l)}}) + (\LFstar + K_A + (K_B+1)L_{\hat{Q}^{(l)}})$
$L_{\hat{Q}^{(l+1)}} \le (2K_A+\LFstar) + (2K_B+1)L_{\hat{Q}^{(l)}}$.
Let $A' = 2K_A+\LFstar$ and $B' = 2K_B+1$. Since $K_A, K_B \ge 0$ and $\LFstar \ge 0$ (as a Lipschitz constant), we have $A' \ge 0, B' \ge 1$.
The recurrence is $L_{\hat{Q}^{(l+1)}} \le A' + B'L_{\hat{Q}^{(l)}}$.
Starting with $L_{\hat{Q}^{(0)}}=0$:
$L_{\hat{Q}^{(1)}} \le A'$
$L_{\hat{Q}^{(2)}} \le A' + B'A' = A'(1+B')$
In general, $L_{\hat{Q}^{(l)}} \le A' \sum_{j=0}^{l-1} (B')^j$.
If $B'=1$ (i.e., $2K_B+1=1 \implies K_B=0$), then $L_{\hat{Q}^{(l)}} \le A'l$.
If $B'>1$, then $L_{\hat{Q}^{(l)}} \le A' \frac{(B')^l-1}{B'-1}$.
The iterates $\hat{Q}^{(l)}$ are considered for $l=0, \dots, L-1$ as inputs to the neural operator blocks. The maximum $l$ here is $L-1$. Thus, their Lipschitz constants are uniformly bounded by $\LhatUnifLip$, defined as:
$$ \LhatUnifLip = \begin{cases} A'(L-1) & \text{if } B'=1 \text{ (and } L>0 \text{)} \\ A' \frac{(B')^{L-1}-1}{B'-1} & \text{if } B'>1 \text{ (and } L>0 \text{)} \\ 0 & \text{if } L=0 \text{ or } L=1 \text{ (input to first layer } \hat{Q}^{(0)} \text{ is 0-Lip)} \end{cases} $$
This $\LhatUnifLip$ is finite since $A', B', L$ are finite, and $\LFstar$ is a fixed constant for the chosen architecture.

We now formally define the target compact set:
$$ \mathcal{K}_{\text{target}} = \{ Q \in C(K_Q) : \supnorm{Q} \le \MhatUnifBound \text{ and } Q \text{ is } \LhatUnifLip\text{-Lipschitz} \}. $$
This set $\mathcal{K}_{\text{target}}$ is compact in $C(K_Q)$ by Arzelà-Ascoli, as argued in Appendix \ref{app:proof_neural_op_uat_fortified}.
We verify by induction that $\hat{Q}^{(l)} \in \mathcal{K}_{\text{target}}$ for $l=0, \dots, L-1$.
\begin{itemize}
    \item Base case ($l=0$): $\hat{Q}^{(0)} \equiv 0$, so $\supnorm{\hat{Q}^{(0)}}=0 \le \MhatUnifBound$, and $L_{\hat{Q}^{(0)}}=0 \le \LhatUnifLip$ (true for $L \ge 1$). Thus $\hat{Q}^{(0)} \in \mathcal{K}_{\text{target}}$.
    \item Inductive step: Assume $\hat{Q}^{(l)} \in \mathcal{K}_{\text{target}}$ for some $l < L-1$. Thus its Lipschitz constant $L_{\hat{Q}^{(l)}} \le \LhatUnifLip$. More precisely, $L_{\hat{Q}^{(l)}} \le A' \frac{(B')^l-1}{B'-1}$ (or $A'l$ if $B'=1$).
    Then Lemma \ref{lem:neural_op_uat} can be applied to $\ModBellmanOp$ with $\hat{Q}^{(l)}$ as input (since $\hat{Q}^{(l)} \in \mathcal{K}_{\text{target}}$). This guarantees existence of $(M, \mathcal{E}_M, \mathcal{D}_M)$ (chosen once based on $\mathcal{K}_{\text{target}}$, $\epsilon_1$) and $\NnOp_{\params_l}$ for $\FfOpTilde_{\params_l}$ such that conditions (a) and (b) above hold (sup-norm error $<\epsilon_1$, output is $\LFstar$-Lipschitz).
    Then $\hat{Q}^{(l+1)} = \BellmanOp \hat{Q}^{(l)} + \delta_l$. We already know $\supnorm{\hat{Q}^{(l+1)}} \le \MhatUnifBound$.
    Its Lipschitz constant $L_{\hat{Q}^{(l+1)}} \le A' + B'L_{\hat{Q}^{(l)}}$.
    Since $L_{\hat{Q}^{(l)}} \le A' \frac{(B')^l-1}{B'-1}$ (or $A'l$ if $B'=1$), it follows that
    $L_{\hat{Q}^{(l+1)}} \le A' \frac{(B')^{l+1}-1}{B'-1}$ (or $A'(l+1)$).
    As $l+1 \le L-1$, this implies $L_{\hat{Q}^{(l+1)}} \le \LhatUnifLip$.
    So $\hat{Q}^{(l+1)} \in \mathcal{K}_{\text{target}}$.
\end{itemize}
This inductive argument confirms that all iterates $\hat{Q}^{(l)}$ for $l=0, \dots, L-1$ (which are inputs to the operator blocks) belong to the same compact set $\mathcal{K}_{\text{target}}$. The choice of $M$ and the properties of the neural operator class (including $\LFstar$) are thus consistently defined for this set. For each layer $l$, we choose specific parameters $\params_l$ for $\NnOp_{\params_l}$ according to Lemma \ref{lem:neural_op_uat} to ensure $\supnorm{\delta_l} < \epsilon_1$.

\bigskip
\noindent\textbf{Step 4: Final Error Bound}.
\medskip
The total error for $\QnnL = \hat{Q}^{(L)}$ is:
$$ \supnorm{\QnnL - \Qstar} \le \supnorm{\hat{Q}^{(L)} - Q^{(L)}} + \supnorm{Q^{(L)} - \Qstar}. $$
This is $e_L + \supnorm{Q^{(L)} - \Qstar}$.
From Step 1, $\supnorm{Q^{(L)} - \Qstar} < \epsilon/2$.
From Step 2 (with justification from Step 3 that $\supnorm{\delta_l} < \epsilon_1$ is achievable for all layers operating on inputs from $\mathcal{K}_{\text{target}}$), $e_L < \epsilon_1 \frac{1}{1-e^{-\lambda\delta}} \le \frac{\epsilon}{2}$.
Therefore,
$$ \supnorm{\QnnL - \Qstar} < \epsilon/2 + \epsilon/2 = \epsilon. $$
This completes the proof.
\end{proof}

\bigskip\bigskip

\section{Further Details on Neural Operators and Function Representations}
\label{app:neural_operators_details}
\medskip

This appendix provides supplementary context on neural operators and function representations relevant to the main arguments, particularly concerning Assumption \ref{assump:neural_operator_class}.

\subsection{Neural Operators and Function Spaces}
\medskip
A neural operator $\tilde{\mathcal{G}}_\theta: \mathcal{X} \to \mathcal{Y}$ is a mapping between function spaces (e.g., $\mathcal{X}=C(D_1)$, $\mathcal{Y}=C(D_2)$), parameterized by neural network weights $\theta$. In our setting, the operator block $\FfOpTilde_{\params_l}: C(K_Q) \to C(K_Q)$ is structured as:
$$ \FfOpTilde_{\params_l}(Q) = \mathcal{D}_M(\NnOp_{\params_l}(\mathcal{E}_M(Q))). $$
The components are:
\begin{itemize}[itemsep=3pt]
    \item $\mathcal{E}_M: C(K_Q) \to \R^M$: An encoding operator, typically point sampling on a fixed grid $D_M = \{p_j\}_{j=1}^M \subset K_Q$, so $\mathcal{E}_M(Q) = (Q(p_1), \dots, Q(p_M))$.
    \item $\NnOp_{\params_l}: \R^M \to \R^M$: A standard neural network, such as a Multi-Layer Perceptron (MLP), parameterized by $\params_l$.
    \item $\mathcal{D}_M: \R^M \to C(K_Q)$: A decoding operator that reconstructs a continuous function from its $M$ discrete values, e.g., using multilinear or higher-order polynomial interpolation, or a learned decoder.
\end{itemize}
\smallskip
Assumption \ref{assump:neural_operator_class} encapsulates the properties needed for our main theorem.
\begin{enumerate}
    \item \textbf{Sup-norm Approximation (Assumption \ref{assump:neural_operator_class} (a)):} This relies on standard UATs for neural operators (e.g., \cite{kovachki2023neural, chen1995universal}), given the continuity of $\ModBellmanOp$ and compactness of $\mathcal{K}_{\text{target}}$. The number of points $M$ and the complexity of $\NnOp_{\params_l}$ depend on $\epsilon_{\text{op}}$ and the properties of $\mathcal{K}$ and $\mathcal{G}(\mathcal{K})$.
    \item \textbf{Lipschitz Property of Approximant Function (Assumption \ref{assump:neural_operator_class} (b)):} This is critical for the stability of Lipschitz constants of $\hat{Q}^{(l)}$. It requires that the constructed function $x \mapsto (\FfOpTilde_{\params_l}(Q))(x)$ itself be Lipschitz on $K_Q$, with a uniformly bounded Lipschitz constant $\LFstar$.
        As detailed in Remark \ref{rem:LFstar_fixed}, $\LFstar$ is determined by the specific choices of $M$, $\epsilon_{\text{op}}$, and the architecture used to implement $\FfOpTilde_{\params_l}$ (e.g., properties of $\NnOp_{\params_l}$ and $\mathcal{D}_M$). The key is that for a given approximation task defined by $\epsilon_{\text{op}}$ on a set $\mathcal{K}$, an architecture can be chosen that both meets the accuracy requirement (a) and yields an output function with a certain Lipschitz constant $\LFstar$. This $\LFstar$ then becomes a parameter of the chosen operator design for that task. For instance, if $\mathcal{D}_M$ uses fixed $L_\phi$-Lipschitz basis functions and $\NnOp_{\params_l}$ is constructed to have outputs bounded by $B_y$ (e.g., through bounded weights or activation functions), then $\LFstar \le M B_y L_\phi$. While $M$ depends on the regularity of $\mathcal{K}$ (e.g., $L_{\mathcal{K}}$) to achieve $\epsilon_{\text{op}}$, the argument is that once $M$ and other architectural constraints ($B_y, L_\phi$) are fixed to satisfy (a), $\LFstar$ is determined by these choices. It is this fixed $\LFstar$ (associated with the specific operator constructed for the approximation task on $\mathcal{K}_{\text{target}}$) that is used in the recurrence for $\LhatUnifLip$.
\end{enumerate}
The existence of such an $\LFstar$ (uniform for $Q \in \mathcal{K}_{\text{target}}$ and $l=0,\dots,L-1$) is essential. This is a stronger requirement than basic UATs for NOs but is considered architecturally achievable by design.
\bigskip
\subsection{Function Representation and Discretization Error}
\medskip
The choice of the discretization grid $D_M$ and the interpolation scheme $\mathcal{D}_M$ directly impacts the reconstruction error term $\supnorm{F - \mathcal{D}_M(\mathcal{E}_M(F))}$ for $F \in \ModBellmanOp(\mathcal{K}_{\text{target}})$.
\begin{itemize}[itemsep=3pt]
    \item If functions $F$ are $L_J(\LhatUnifLip)$-Lipschitz (as established for $F \in \ModBellmanOp(\mathcal{K}_{\text{target}})$) on $K_Q$ (a compact domain of dimension $d_Q = 1+n+m$), then standard multilinear interpolation on a uniform grid $D_M$ with mesh size $h_g \sim M^{-1/d_Q}$ yields an error of $O(L_J(\LhatUnifLip) h_g)$.
    \item If functions possess higher smoothness (e.g., are $C^s$ with bounded $s$-th derivatives), the error can be $O(h_g^s)$.
\end{itemize}
To achieve an interpolation error component of $\epsilon_{\text{interp}}$, $M$ might need to be of the order $(L_J(\LhatUnifLip)/\epsilon_{\text{interp}})^{d_Q/s}$ (where $s=1$ for Lipschitz functions). This dependence on $d_Q$ highlights the curse of dimensionality (CoD) concerning the required number of grid points $M$. While the CoD does not invalidate the UAT's existence claim (which is generally non-quantitative in $M$ for generic NO-UATs unless specific rates are proven for the operator class and function regularity), it underscores practical challenges for approximating functions in high-dimensional $(t,s,a)$ spaces using grid-based methods. This is further discussed in Section \ref{sec:quantitative_rates}.

\bigskip
\subsection{Example UAT for Neural Operators}
\medskip
References such as \cite{chen1995universal} or \cite{kovachki2023neural} provide the theoretical foundation for Lemma \ref{lem:neural_op_uat} (a).
\begin{itemize}[itemsep=3pt]
    \item For instance, \cite{kovachki2023neural} states that for a continuous operator $\mathcal{G}: \mathcal{A} \to C(D_Y; \R^{d_y})$, where $\mathcal{A}$ is a compact subset of $C(D_X; \R^{d_x})$, and for any $\epsilon > 0$, there exists a neural operator $\mathcal{G}_\theta$ (belonging to a specific class, e.g., Graph Kernel Network based, or FNO-like) such that $\sup_{u \in \mathcal{A}} \|\mathcal{G}(u) - \mathcal{G}_\theta(u)\|_{C(D_Y)} < \epsilon$.
    \item The architecture $\FfOpTilde_{\params_l}(Q) = \mathcal{D}_M(\NnOp_{\params_l}(\mathcal{E}_M(Q)))$ can be viewed as a specific realization of such a neural operator. This is particularly clear when $\mathcal{D}_M$ employs a set of basis functions and $\NnOp_{\params_l}$ computes the coefficients for these basis functions. The DeepONet architecture \cite{lu2021learning} also fits this general framework and has its own UAT.
\end{itemize}
Assumption \ref{assump:neural_operator_class} (b) extends these typical UATs by requiring control over the Lipschitz constant of the output function $(\FfOpTilde_{\params_l}(Q))(x)$, not just the Lipschitz constant of the operator $\FfOpTilde_{\params_l}$ itself (mapping function to function in sup-norm). This is justified by constructing $\NnOp_{\params_l}$ to be Lipschitz as a map $\R^M \to \R^M$ and to have bounded outputs, and choosing a decoder $\mathcal{D}_M$ that translates these properties into a controlled Lipschitz constant for the resulting function on $K_Q$.

\section{Illustrative Example of Bellman Iteration and Regularity}
\label{app:illustrative_example}

This appendix provides a highly simplified example to illustrate the iterative application of the Bellman operator and the concept of regularity propagation, as discussed in Section \ref{sec:q_fbsde_scheme} and Lemma \ref{lem:bellman_op_props}. This example is purely pedagogical and does not involve neural network approximation, but its dynamics can be visualized.

Consider a one-dimensional continuous state space $\StateSpace = [0,1]$.
Let the action space be discrete, $\ActionSpace = \{a_L, a_R\}$, representing "move left" and "move right".
The dynamics are deterministic. For a state $s \in [0,1]$ and a small step size $\Delta s = 0.1$:
\begin{itemize}
    \item Action $a_L$: $s' = \max(0, s - \Delta s)$.
    \item Action $a_R$: $s' = \min(1, s + \Delta s)$.
\end{itemize}
Let the immediate reward be $r(s,a) = r(s) = -(s-0.5)^2$. This reward is bounded ($|r(s)| \le 0.25$) and Lipschitz continuous on $[0,1]$ (e.g., $|r'(s)| = |-2(s-0.5)| \le 1$, so $L_r=1$).
Let the discount factor be $\gamma = 0.9$.
We consider a finite horizon of $N=2$ "steps to go". The value functions $Q^{(k)}(s,a)$ are iterates, where $k$ represents the number of steps of Bellman updates performed, starting from an initial guess. $Q^{(0)}$ is the initial guess, and $Q^{(N)}$ is the Q-function after $N$ updates. For this example, $N=2$, so we compute $Q^{(0)}, Q^{(1)}, Q^{(2)}$. Here, $Q^{(2)}$ will be the optimal Q-function for a problem that lasts two stages from the current decision point.

The Bellman iteration for $Q^{(k+1)}(s,a)$ from $Q^{(k)}$ is:
$$ Q^{(k+1)}(s,a) = r(s) + \gamma \max_{a' \in \ActionSpace} Q^{(k)}(s'(s,a), a') $$
where $s'(s,a)$ is the state resulting from taking action $a$ in state $s$.

\textbf{Iteration $k=0$ (Initial Q-function):}
$$ Q^{(0)}(s,a) = 0 \quad \text{for all } s \in [0,1], a \in \{a_L, a_R\}. $$
This function is trivially bounded (by $M_0=0$) and Lipschitz continuous (with $L_0=0$). Numerical computation confirms $Q^{(0)}(s,a)$ is identically zero (Figure \ref{fig:bellman_illustration_appendix}, top panel).

\textbf{Iteration $k=1$:}
Since $\max_{a'} Q^{(0)}(s'(s,a), a') = 0$, we have:
$$ Q^{(1)}(s,a) = r(s) + \gamma \cdot 0 = -(s-0.5)^2. $$
Note that $Q^{(1)}(s,a)$ is independent of $a$ in this case.
\begin{itemize}
    \item \textbf{Boundedness:} $\supnorm{Q^{(1)}} = 0.25$.
    \item \textbf{Lipschitz Continuity:} As $r(s)$ is $L_r=1$-Lipschitz, $Q^{(1)}(s,a)$ is $1$-Lipschitz with respect to $s$. It is constant (and thus $0$-Lipschitz) with respect to $a$.
\end{itemize}
Numerical computation confirms $Q^{(1)}(s,a) = -(s-0.5)^2$ with negligible error (max absolute difference of $0.00 \times 10^{0}$ from the analytic form). This is visualized in Figure \ref{fig:bellman_illustration_appendix} (middle panel), where both action curves overlap.

\textbf{Iteration $k=2$ (Q-function after two updates):}
Let $s_L'(s) = \max(0, s - \Delta s)$ and $s_R'(s) = \min(1, s + \Delta s)$.
The term $\max_{a' \in \ActionSpace} Q^{(1)}(s'(s,a), a')$ becomes $V^{(1)}(s'(s,a))$, where $V^{(1)}(s) = \max_{a^*} Q^{(1)}(s,a^*) = -(s-0.5)^2$.
So, $Q^{(2)}(s,a)$ is:
\begin{align*}
    Q^{(2)}(s,a_L) &= r(s) + \gamma V^{(1)}(s_L'(s)) = -(s-0.5)^2 + \gamma \left( - (s_L'(s) - 0.5)^2 \right) \\
    Q^{(2)}(s,a_R) &= r(s) + \gamma V^{(1)}(s_R'(s)) = -(s-0.5)^2 + \gamma \left( - (s_R'(s) - 0.5)^2 \right)
\end{align*}
\begin{itemize}
    \item \textbf{Boundedness:} Since $r(s)$ and $V^{(1)}(s)$ are bounded, $Q^{(2)}(s,a)$ is clearly bounded. For example, $\supnorm{Q^{(2)}} \le \supnorm{r} + \gamma \supnorm{V^{(1)}} = 0.25 + \gamma \cdot 0.25 = 0.25(1+\gamma)$.
    \item \textbf{Lipschitz Continuity (w.r.t. $s$):}
        The function $s \mapsto s_L'(s)$ is $1$-Lipschitz. The function $x \mapsto V^{(1)}(x)$ is $1$-Lipschitz on $[0,1]$. The composition $s \mapsto V^{(1)}(s_L'(s))$ is $1$-Lipschitz.
        Thus, $Q^{(2)}(s,a_L)$ is a sum of $r(s)$ ($1$-Lipschitz) and $\gamma V^{(1)}(s_L'(s))$ ($\gamma \cdot 1$-Lipschitz). Its Lipschitz constant w.r.t. $s$ is bounded by $1+\gamma$. A similar argument holds for $Q^{(2)}(s,a_R)$.
\end{itemize}
Numerical computation yields $Q^{(2)}$ values that match these analytical forms with very small differences (max absolute difference of approximately $5.55 \times 10^{-17}$), primarily due to interpolation of $V^{(1)}(s')$ values when $s'$ falls between discretized state points. The distinct shapes for $Q^{(2)}(s,a_L)$ and $Q^{(2)}(s,a_R)$ are visible in Figure \ref{fig:bellman_illustration_appendix} (bottom panel).

\begin{figure}[htbp]
    \centering
    \includegraphics[width=0.7\textwidth]{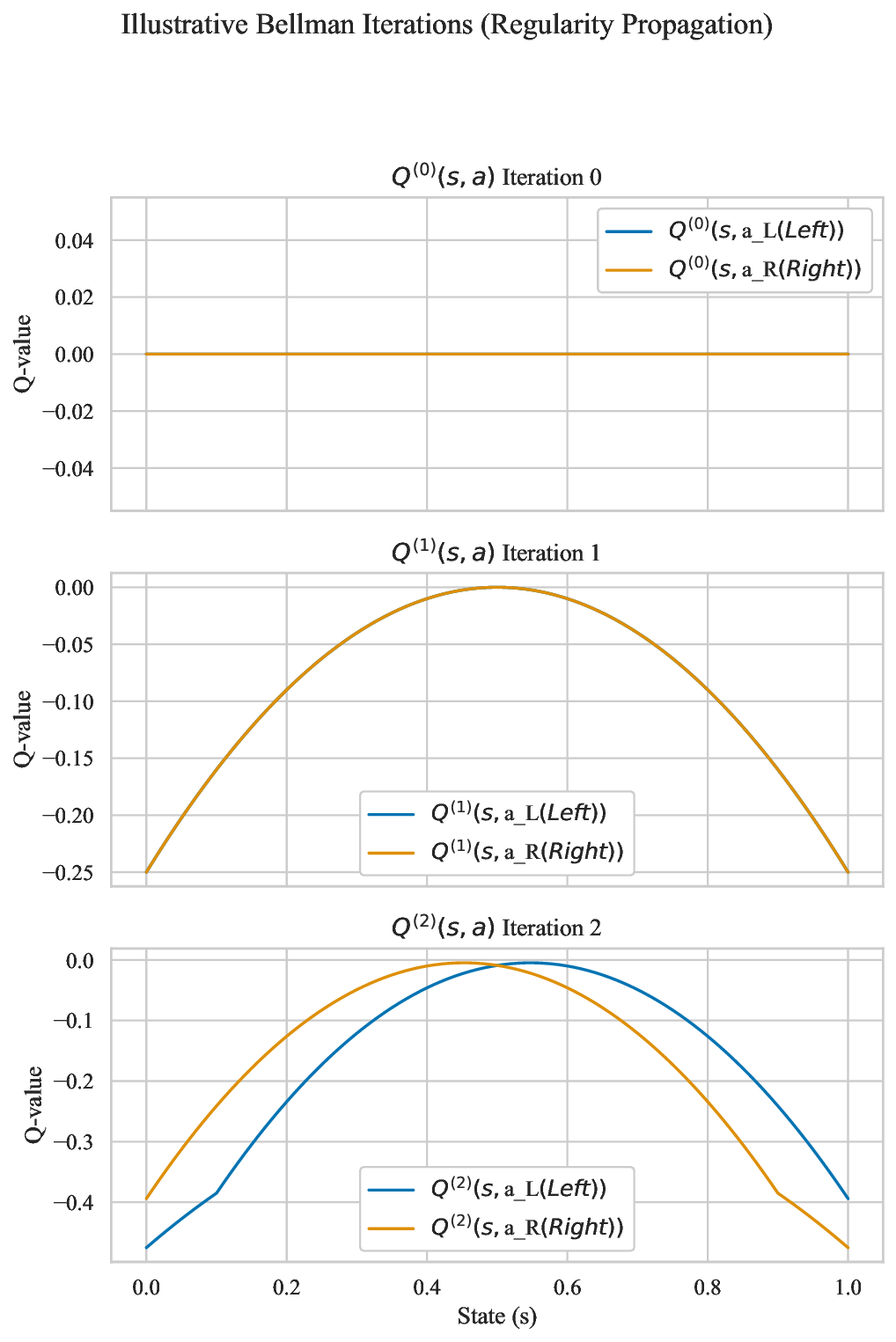}
    \caption{Illustration of Bellman iterates $Q^{(k)}(s,a)$ for $k=0, 1, 2$. The state space is $s \in [0,1]$. Actions $a_L$ (move left) and $a_R$ (move right) are shown. $Q^{(0)}$ is zero. $Q^{(1)}$ is identical for both actions. $Q^{(2)}$ shows distinct values for $a_L$ and $a_R$, reflecting the one-step lookahead with $V^{(1)}$. All iterates are bounded and visually Lipschitz continuous.}
    \label{fig:bellman_illustration_appendix}
\end{figure}

This simple example, supported by the numerical results and Figure \ref{fig:bellman_illustration_appendix}, demonstrates:
\begin{enumerate}[noitemsep]
    \item The iterative nature of the Bellman operator, transforming an initial Q-function estimate.
    \item That if $Q^{(k)}$ is bounded and Lipschitz, then $Q^{(k+1)}$ remains bounded and Lipschitz. The visual smoothness and boundedness of the functions in Figure \ref{fig:bellman_illustration_appendix} are consistent with this. This aligns with the findings of Lemma \ref{lem:bellman_op_props}(d).
    \item The functions $Q^{(k)}(s,a)$ are well-behaved (continuous, Lipschitz) under standard assumptions on rewards and dynamics, even with the $\max$ operations.
\end{enumerate}
While highly simplified, it captures the essence of the function sequence $Q^{(k)}$ converging to $\Qstar$ within a space of regular functions.

\end{document}